\documentclass[journal,twoside]{IEEEtran}

\usepackage{graphicx}		
\usepackage{algorithm}		
\usepackage{algpseudocode}	
\usepackage{amsmath} 		
\usepackage{amssymb}		
\usepackage{amsthm}		    
\usepackage{array}			
\usepackage{bm}			    
\usepackage{comment}		
\usepackage{multirow}		
\usepackage{multicol}		
\usepackage{subfigure}		
\usepackage{hyperref} 		
\usepackage{fancyhdr} 		
\usepackage{lastpage} 		
\usepackage{xcolor}			
\usepackage{dsfont}			
\usepackage{hhline}			
\usepackage{cite}			

\usepackage[latin1]{inputenc}	
\usepackage{tikz}
\usetikzlibrary{calc}

\usepackage[font=footnotesize, justification=justified, labelsep=period, format=plain]{caption} 

\definecolor{blue_color}{rgb}{0.0, 0.0, 0.6}

\definecolor{green_color}{rgb}{0.0, 0.6, 0.0}

\definecolor{red_color}{rgb}{0.8, 0.0, 0.0}
\newcommand{\red}[1]{\textcolor{red_color}{#1}}
\definecolor{orange_color}{rgb}{0.9, 0.5, 0.2}

\newcommand{\Reals}[1]{\mathbb{R}^{#1}}	
\newcommand{\norm}[1]{\|{#1}\|}			
\newcommand{\x}{\mathbf{x}}				
\renewcommand{\t}{\mathbf{t}}				
\newcommand{\Zero}[1]{\mathbf{0}_{#1}}		
\newcommand{\Eye}[1]{\mathbf{I}_{#1}}		
\newcommand{\One}[1]{\mathbf{1}_{#1}}		
\newcommand{\tr}{\mathrm{tr}}				
\newcommand{\Null}{\mathrm{Null}}			
\newcommand{\R}{\mathbf{R}}				
\newcommand{\Neighs}[1]{\mathcal{N}_{#1}}   
\newcommand{\Ind}{\mathds{1}}			
\DeclareMathOperator*{\minimize}{minimize}	

\newcommand{\A}{\mathbf{A}}					
\renewcommand{\u}{\mathbf{u}}				

\newtheorem{theorem}{Theorem}
\newtheorem{lemma}{Lemma} 
\newtheorem{corollary}{Corollary}

\newtheorem{definition}{Definition}

\newtheorem{problem}{Problem}

\graphicspath{{./figures/}}

\begin{document}
\title{\textbf{GeoD:} Consensus-based \textbf{Geo}desic\\ \textbf{D}istributed Pose Graph Optimization}
\author{Eric~Cristofalo, 
        Eduardo~Montijano, 
        and 
        Mac~Schwager
\thanks{This work was supported by a NDSEG fellowship; Spanish projects PGC2018-098817-A-I00 (MINECO/FEDER), DGA T04-FSE; NSF grants CNS-1330008 and IIS-1646921;  ONR grant N00014-18-1-2830 and ONRG-NICOP-grant N62909-19-1-2027; and the Ford-Stanford Alliance Program.}%
\thanks{E. Cristofalo and M. Schwager are with the Department of Aeronautics and Astronautics, Stanford University, Stanford, CA 94305, USA {\tt\scriptsize (ecristof;schwager@stanford.edu)}}%
\thanks{E. Montijano is with Instituto de Investigaci\'on en Ingenier\'ia de Arag\'on, Universidad de Zaragoza, Zaragoza 50018, Spain {\tt\scriptsize (emonti@unizar.es)}.}%
}


\maketitle

\begin{abstract}
We present a consensus-based distributed pose graph optimization algorithm for obtaining an estimate of the 3D translation and rotation of each pose in a pose graph, given noisy relative measurements between poses. The algorithm, called GeoD, implements a continuous time distributed consensus protocol to minimize the geodesic pose graph error. GeoD is distributed over the pose graph itself, with a separate computation thread for each node in the graph, and messages are passed only between neighboring nodes in the graph. We leverage tools from Lyapunov theory and multi-agent consensus to prove the convergence of the algorithm. We identify two new consistency conditions sufficient for convergence: \textit{pairwise consistency} of relative rotation measurements, and \textit{minimal consistency} of relative translation measurements. GeoD incorporates a simple one step distributed initialization to satisfy both conditions. We demonstrate GeoD on simulated and real world SLAM datasets. We compare to a centralized pose graph optimizer with an optimality certificate (SE-Sync) and a Distributed Gauss-Seidel (DGS) method. On average, GeoD converges 20 times more quickly than DGS to a value with 3.4 times less error when compared to the global minimum provided by SE-Sync. GeoD scales more favorably with graph size than DGS, converging over 100 times faster on graphs larger than 1000 poses. Lastly, we test GeoD on a multi-UAV vision-based SLAM scenario, where the UAVs estimate their pose trajectories in a distributed manner using the relative poses extracted from their on board camera images. We show qualitative performance that is better than either the centralized SE-Sync or the distributed DGS methods. 
\center{\href{https://github.com/ericcristofalo/GeoD}{github.com/ericcristofalo/GeoD}}
\end{abstract}


\section{Introduction}
\label{sec:introduction}

In this paper, we propose a distributed algorithm for pose graph optimization called \textbf{Geo}desic \textbf{D}istributed Pose Graph Optimization, or GeoD. Our algorithm uses rotation errors computed from the geodesic distance on the manifold of rotations. This is in contrast to existing methods, which typically use an off-manifold rotation error known as the chordal distance. Our algorithm is based on multi-agent consensus, and is therefore distributed over the topology of the pose graph itself. Each pose in the pose graph is optimized through simple iterative update rules that implement a continuous time gradient descent on the geodesic pose graph error. Using tools from Lyapunov theory, we analytically prove that GeoD converges to a set of equilibria corresponding to local minima of the pose graph error. We show that GeoD gives fast and accurate empirical performance in comparisons with state-of-the-art centralized and distributed pose graph optimization algorithms. 

\begin{figure} [t]
\centering
	\includegraphics[width=1.0\linewidth]{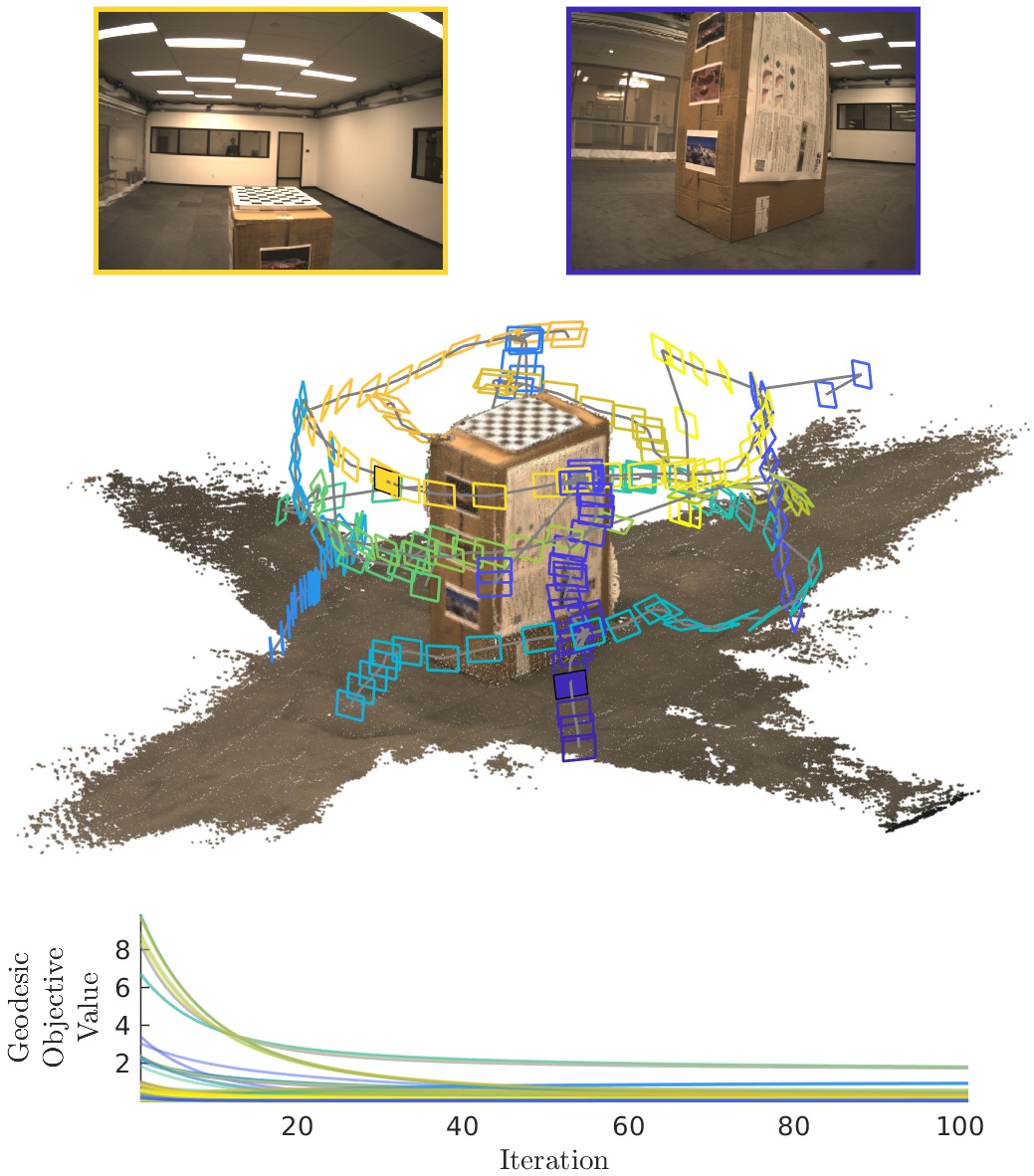}
\caption{We present GeoD, a consensus-based distributed $SE(3)$ pose graph optimization algorithm with provable convergence guarantees. This algorithm enables a group of robots, given noisy relative pose measurements, to reach agreement on the group's $SE(3)$ pose history in a distributed manner. The figure shows GeoD applied to a group of 7 quadrotor robots. (Top) Two images showing onboard views used to compute relative poses. (Middle) The reconstructed 3D map and pose history of all quadrotors. (Bottom) The geodesic objective value over time, showing convergence of the estimate for all quadrotor poses.}
\label{fig:main}
\end{figure}

Our algorithm has the potential to enable systems with hundreds or thousands of robots to cooperatively maintain large-scale SLAM maps of an area. For example, a fleet of autonomous vehicles could use GeoD to maintain an up-to-date distributed SLAM map of an entire city, incorporating the sensor measurements of all vehicles, while distributing the storage and computation among the vehicles. Similarly, a fleet of delivery drones could use GeoD to cooperatively maintain an estimate of the locations of all drones in the area, while maintaining a 3D map of a neighborhood for making safe door-step deliveries. In such safety critical multi-robot systems, mathematical assurances that the algorithm will not diverge or oscillate in an unstable fashion are essential. GeoD provides such convergence assurances.

While the majority of pose graph optimization algorithms are centralized~\cite{CadenaEtAlTRO16PastPresentFutureSLAM,kummerle2011g,rosen2016se}, GeoD is distributed, anytime, and consensus-based. Unlike other distributed pose graph optimization algorithms \cite{choudhary2017distributed}, GeoD is distributed over the pose graph itself, so there is a different computational thread responsible for optimizing each pose in the graph. Each thread only passes messages with other threads that are neighbors in the pose graph. This architecture is flexible in that the multiple threads can be partitioned among multiple robots as a distributed multi-robot SLAM back-end. Each robot implements its own threads and communicates over a wireless network to pass messages with threads owned by other robots. In the case of a single robot, our algorithm also can be implemented on a GPU to give a fast, parallelized SLAM back-end. 

GeoD operates on the geodesic distance on $SO(3)$ (the Special Orthogonal group of $3$D rotation matrices) as opposed to using convex chordal relaxations, which is the state-of-the-art in the literature~\cite{kummerle2011g,rosen2016se,choudhary2017distributed}.
We leverage tools from Lyapunov theory and multi-agent consensus to formally prove that GeoD will converge on an estimate. We introduce the concept of \textit{measurement consistency} to define the conditions for convergence which have not previously been identified in the consensus-based pose estimation literature~\cite{tron2009distributed, thunberg2014distributed, thunberg2017distributed}. 
Given arbitrary noisy relative pose measurements, GeoD starts with a single-step distributed initialization round to satisfy these conditions for convergence. 

We illustrate the advantages of GeoD by comparing it to state-of-the-art centralized (SE-Sync~\cite{rosen2016se}) and distributed (DGS~\cite{choudhary2017distributed}) chordal relaxation approaches. 
On average over several SLAM  datasets, GoeD converges to a solution more than $700\times$ faster than SE-Sync with only $32.5\%$ error from the global minimum on chordal value --- despite its gradient-like nature --- compared to $110.2\%$ error obtained by DGS. 
Our method requires similar computation time to DGS for small networks, but requires approximately $100\times$ less time for large networks that contain more than $1000$ poses. 
We finally demonstrate our algorithm in hardware experiments with a multi-UAV network, where the camera-equipped drones optimize a SLAM map using noisy relative pose measurements extracted from shared image features (see Fig.~\ref{fig:main}). 


This work builds upon an earlier conference version~\cite{cristofalo2019consensus}, which used the angle-axis rotation parameterization. The current paper presents a fundamentally new algorithm using rotation matrices, together with an entirely new mathematical analysis. We also present new performance comparisons with two state-of-the-art pose graph optimization algorithms, and give results from new hardware experiments. We briefly compare GeoD with the earlier iteration from \cite{cristofalo2019consensus} in Section~\ref{sec:results} and show that GeoD is advantageous in terms of both convergence speed as well as solution accuracy by over a factor of $5\times$. 

The rest of this paper is organized as follows. We review the pose graph optimization literature in Section~\ref{sec:related_work}. We present the problem formulation and overview our solution in Section~\ref{sec:problem}. We briefly review the group of $3$D rotation matrices in Section~\ref{sec:lie_group}. We prove convergence of GeoD in Section~\ref{sec:analysis}. Lastly, we present comparisons to state-of-the-art centralized and distributed approaches on SLAM datasets and a multi-agent system experiment in Section~\ref{sec:results}. Finally, we draw conclusions in Section~\ref{sec:conclusion}.

\section{Related Work}
\label{sec:related_work}

Pose graph optimization refers to a standard class of optimization problems in which a set of SE(3) rigid body poses (usually corresponding to poses for a robot, or robots, over time) must be estimated based on noisy relative measurements between these poses. Each pose represents a node and each relative measurement represents an edge, hence the underlying structure is represented by a graph. Pose graph optimization is fundamental to Simultaneous Localization and Mapping (SLAM), as it typically forms the ``back-end'' of a SLAM system, while the ``front-end'' produces the edges in the pose graph by extracting relative pose estimates from raw sensor measurements \cite{CadenaEtAlTRO16PastPresentFutureSLAM,engel2014lsd,mur2015orb}. While mostly studied in the context of SLAM, pose graph optimization can also be used as an underlying optimization model for multi-robot relative localization \cite{MartinelliICRA05RelativeLocalization}, trajectory planning \cite{ToussaintICML09PlanningInference}, structure from motion \cite{DellaertCVPR00SFM}, and a variety of other problems fundamental to robotics \cite{DellaertFoundationaAndTrends17FactorGraphs}.

Pose graph optimization is a nonlinear, nonconvex, least squares optimization for a translation vector and rotation matrix corresponding to each pose in the network.
Advances in bundle adjustment have led to very efficient gradient methods~\cite{duckett2002fast} or Levenberg-Marquardt algorithms~\cite{kummerle2011g} for pose graph optimization. 
Most recently, the authors of~\cite{rosen2016Certifiably,rosen2016se} introduced a semi-definite relaxation to the pose graph optimization problem that can be efficiently solved using an algorithm called SE-Sync. This algorithm, under broad conditions, gives a certificate of optimality in the following sense: if the certificate is produced, the solution is globally optimal, if not, the solution may or may not be global optimal. 

These previous pose graph optimization methods are centralized and do not immediately lend themselves to decentralization. Distributed pose graph optimization is important for applications where multiple robots must collaborate to estimate a common SLAM map. Distributed algorithms for pose graph optimization are also important for taking advantage of multiple cores on a single chip, parallelization in a GPU, or in maintaining very large SLAM maps over many machines in a data center, as in cloud robotics applications. There are numerous multi-robot SLAM solutions that consider the distributed maximum likelihood estimation problem for linear Gaussian systems~\cite{xiao2005scheme} and 3D nonlinear systems~\cite{indelman2014multi}. 
There are also distributed Kalman filter-like approaches~\cite{olfati2005distributed}, filtering methods using vision-specific sensors~\cite{nageli2014environment}, and distributed SLAM with cloud-based resources~\cite{riazuelo2014c2tam}. 
Like the centralized approaches, distributed pose graph optimizers similarly rely on a convex chordal relaxation, but employ either distributed Successive Over-Relaxation or Jacobi Over-Relaxation solvers~\cite{choudhary2017distributed}. Specifically,~\cite{choudhary2017distributed} presents a two-stage algorithm for Distributed Gauss-Seidel (DGS) optimization that first distributes the chordal initialization of rotations, and then distributes a single Gauss-Newton step of the full pose graph optimization problem. In contrast, our method optimizes the full $SE(3)$ state in a single algorithm and does not require a chordal relaxation. 
A recent distributed pose graph solver~\cite{tian2019block} provides a certificate of optimality using a sparse semidefinite relaxation, similarly to SE-Sync. Finally, another recent solver is asynchronous and distributed~\cite{tian2020asynchronous}, but requires the pose graph updates to be executed separately, considering a single robot at a given time.
Differently, our algorithm is able to update all the graph poses simultaneously in the fashion of consensus-based estimation methods.
In this paper we compare GeoD to SE-Sync~\cite{rosen2016se}, as a representative of a state-of-the-art centralized and certifiable pose graph solver, and DGS~\cite{choudhary2017distributed}, as a representative of a state-of-the-art distributed pose graph solver. 

Our algorithm is based on multi-agent consensus, which are a family of continuous time control algorithms for a group of computational agents to agree on a common variable of interest~\cite{olfati2004consensus}. Consensus-based algorithms are advantageous because they only require communication among local neighbors in a graph. These algorithms are generally light-weight, simple to implement, and can come with formal guarantees of convergence. Our work builds upon consensus algorithms in the multi-robot control literature~\cite{montijano2011multi, wang2016multi,montijano2016vision}, where robots specifically reach consensus on their relative positions while sharing raw measurements. 

Prior work on distributed consensus-based pose estimation can thus be split into work that considers either relative pose measurements corrupted by noise or not corrupted by noise. Without noise, classic consensus algorithms~\cite{olfati2004consensus} are well-suited for translation-only estimation. A number of methods have been proposed for orientation synchronization without noise that consider the angle-axis representation~\cite{thunberg2014distributed}, QR-factorization~\cite{thunberg2018dynamic},  synchronization on the manifold of $SO(3)$~\cite{tron2012intrinsic}, or rotation averaging~\cite{moakher2002means}. A common-frame $SE(3)$ estimation method is proposed in~\cite{tron2008distributed}, where all agents desire to reach consensus on a common frame given a relative $SE(3)$ pose from that frame. 

Noisy relative measurements have largely been considered in the analysis of translation-only estimation, a form of distributed linear least squares estimation~\cite{garin2010survey}. 
For example,~\cite{anderson2010formal} considers distance-only measurements while~\cite{savvides2005analysis} considers bearing only measurements. 
Noisy relative translation measurements themselves are considered in~\cite{xiao2005scheme}, which discusses the distributed computation of the maximum likelihood translation estimate. 
Other methods solve for the centroid of the network and anchor-based translations simultaneously~\cite{aragues2012distributed}. 
Asynchronous estimation is treated in~\cite{carron2014asynchronous} and $2$D estimation with Gaussian noise is analyzed in~\cite{piovan2013frame}. 
There are even fewer solutions that consider distributed $3$D pose estimation given noisy relative measurements. 
Most notably, the distributed gradient descent method in~\cite{tron2014distributed} sequentially considers three separate objectives: first a convex relaxation of the rotation-only objective, then a translation-only objective, and finally the full $3$D pose objective. 
Recently, the distributed $SE(3)$ synchronization method in~\cite{thunberg2017distributed} guarantees convergence with 
an iterative algorithm that projects the rotation estimates at each step onto $SO(3)$. 
These works consider distributed estimation for $3$D pose, but do not characterize and analyze the conditions on the noisy measurements for consensus-based estimation systems to converge as we do in Section~\ref{sec:analysis}. Finally, the bulk of the existing methods are not deployed on real robotic networks in SLAM applications as our method is in Section~\ref{sec:results}. 
\section{Problem Formulation and Solution Overview}
\label{sec:problem}

The goal of this paper is to compute the global translations and rotations of each pose in a pose graph to minimize a translation and rotation error with respect to the noisy relative measurements. Specifically, we aim to solve Problem~\ref{prob:main}, 
\begin{problem}[Pose Graph Optimization Problem]
\label{prob:main}
\begin{equation*} \begin{tiny}
	\minimize_{ \substack{\t_1,\hdots,\t_n \\ \R_1,\hdots,\R_n} } 
	\quad 
	\sum_{i \in \mathcal{V}} \sum_{j\in\Neighs{i}}
	\| \t_j - \t_i - \R_i\tilde{\t}_{ij} \|_2^2 + 
	\| \log ( \R_i^T\R_j\tilde{\R}_{ij}^T )^{\vee} \|_2^2
	\, .
\end{tiny} \end{equation*}
\end{problem}
\noindent
The output of the optimization are the $n$ poses\footnote{Each \textit{pose} here can represent either i) the pose of one robot or multiple robots that are evolving over time (as in SLAM) or ii) the poses of a group of robots at a given time (as in multi-robot relative localization). } that are labeled by the vertex set $\mathcal{V} = \{1,\hdots,n\}$. 
The $3$D rigid body pose of $i \in \mathcal{V}$ is measured with respect to an arbitrary global coordinate frame as a tuple containing a translation vector and a rotation matrix, $\left(\t_i \in \Reals{3}, \R_i \in SO(3)\right) \in SE(3)$, or the $3$D Special Euclidean group. Here $SO(3) = \{ \R \in \Reals{3\times 3} \mid \R^T\R = \Eye{3}, \ \det(\R) = +1 \}$ is the $3$D Special Orthogonal group where $\Eye{3}$ is a $3\times 3$ identity matrix and $\det(\cdot)$ is the determinant of a matrix. 

The noisy relative measurements are defined over the undirected edge set $\mathcal{E}$. 
Each noisy measurement is represented the directed edge component $(i, j)$ and is defined as the tuple $(\tilde{\t}_{ij}, \tilde{\R}_{ij}) \in SE(3)$. 
This notation indicates the pose of $j$ measured in the coordinate frame of pose $i$. 
In an undirected graph, we assume there exists another unique noisy measurement in the opposite direction (of pose $i$ in frame $j$) such that $\tilde\R_{ij}\tilde\R_{ji} \neq \Eye{3}$, for example.
The noisy measurements are assumed to be acquired prior to solving the problem and are not measured at each time-step of the solution as in a visual odometry (VO) or online SLAM framework. 
The vertex set and the edge set define the graph $\mathcal{G} = (\mathcal{V}, \mathcal{E})$. Finally, each pose has a local neighborhood subgraph, $\Neighs{i} \subset \mathcal{V}$, that defines each pose's immediately connected neighbor labels. 


The rotation error term in Problem~\ref{prob:main} is the \textit{geodesic distance} between the relative rotation ($\R_i^T\R_j$) and the relative rotation measurement ($\tilde\R_{ij}$): $\| \log ( \R_i^T\R_j\tilde{\R}_{ij}^T )^{\vee} \|_2$. Geodesics are distances along the manifold and require additional machinery from Lie group theory. We will briefly review concepts related to the Lie group structure of $SO(3)$ in Section~\ref{sec:lie_group}. The use of the geodesic distance is in contrast to many recent works that instead consider the \textit{chordal distance}~\cite{rosen2016Certifiably,rosen2016se,choudhary2017distributed}. The chordal distance, $\| \R_i^T\R_j - \tilde{\R}_{ij}\|_{F}$, computes the Frobenius matrix norm of the difference in rotation matrices. This objective is used in \textit{chordal relaxation} methods --- such as rotation initialization techniques~\cite{carlone2015initialization} --- that formulate a convex optimization to solve for the rotations embedded in a vector space $(\Reals{9})$, and then projects the final result back to $SO(3)$. A great comparison of the geodesic and chordal distance for rotation averaging can be found in~\cite{moakher2002means}. We present a detailed comparison to state-of-the-art chordal relaxation methods in Section~\ref{sec:results}. 

Our distributed consensus-based $SE(3)$ method is conceptually similar to distributed formation control, where the global translation and rotation estimates of each individual pose $i$ evolve according to rigid body kinematics, 
\begin{subequations} 
\label{eq:estimation_dynamics}
\begin{align}
	\dot{\t}_i &= \bm{\nu}_i \, , \\
	\dot{\R}_i &= \R_i \bm{\omega}_i^{\wedge}
	\, .
\end{align} \end{subequations}
The \textit{control inputs} to the kinematic system are the linear velocity $\bm{\nu}_i\in\Reals{3}$ and angular velocity $\bm{\omega}_i\in\Reals{3}$ for the rigid body pose. The \textit{hat} operator $(\cdot)^{\wedge}$ maps a vector into its skew-symmetric matrix analog (more detail is provided in Section~\ref{sec:lie_group}). In particular, we propose to use the following feedback input laws for pose graph optimization, 
\begin{subequations}
\label{eq:estimation_control}
\begin{align}
\label{eq:estimation_control_translation}
	\bm{\nu}_i &= \sum_{j\in\Neighs{i}} \t_j - \t_i - \R_i \tilde\t_{ij} \, , \\
\label{eq:estimation_control_rotation}
	\bm{\omega}_i^{\wedge} &= \sum_{j\in\Neighs{i}} \log(\R_i^T\R_j \tilde\R_{ij}^T)
	\, .
\end{align} \end{subequations}
Note that the control terms take the same form as the error terms in the objective of Problem~\ref{prob:main}. 
The full $SE(3)$ distributed estimation consists of executing the kinematics of system~\eqref{eq:estimation_dynamics} with control inputs~\eqref{eq:estimation_control}, i.e., 
\begin{subequations}
\label{eq:estimation_system}
\begin{align}
\label{eq:translation_estimation_system}
	\dot{\t}_i &= \sum_{j\in\Neighs{i}} \t_j - \t_i - \R_i \tilde\t_{ij} \, , \\
\label{eq:rotation_estimation_system}
	\dot{\R}_i &= \R_i \left( \sum_{j\in\Neighs{i}} \log(\R_i^T\R_j \tilde\R_{ij}^T) \right)
	\, .
\end{align} \end{subequations}
In the rest of this paper, we refer to system~\eqref{eq:estimation_system} as the \textbf{Geo}desic \textbf{D}istributed pose graph optimization system, or GeoD. GeoD is nonconvex for the same reasons as in Problem~\ref{prob:main}. We also highlight that even though system~\eqref{eq:estimation_system} looks like the first-order dynamics of a physical rigid body, we are using it to describe the nonlinear rate of change of our $SE(3)$ pose estimation. Lastly, the two expressions in~\eqref{eq:estimation_system} are assumed to run simultaneously, which is a notable difference compared to~\cite{tron2014distributed} which sequentially considers rotation, translation, and then the full $SE(3)$ problem. 

It is worth noting that the estimation scheme we propose is similar to other distributed consensus-based algorithms. In particular, it is inspired by previous $3$D distributed pose estimation work~\cite{cristofalo2019consensus} and by $3$D formation control ~\cite{montijano2016vision}. Note that there are three major differences. First, with respect to~\cite{cristofalo2019consensus}, we use the rotation matrix parameterization of rotations instead of the angle-axis. This allows us to converge to a solution that is closer to the global minimum as we demonstrate in Section~\ref{sec:results_distributed}. 
Second, the measurements we consider are corrupted by noise and no longer encode a \textit{desired formation} as in the formation control literature. 
Consensus-based estimators with perturbed measurements are not guaranteed to converge to a stable equilibria set in general~\cite{garin2010survey}.
We show conditions on $(\tilde\t_{ij}, \tilde\R_{ij})$ for convergence of this system in Section~\ref{sec:analysis}. 
Third, the measurements are observed once at the beginning of the estimation and not at each step in the estimation. 
This way, the noise is not fed back into the system at each time step. 

Finally, we define a few additional quantities to support the development in the next sections. 
The graph Laplacian $\mathbf{L} \in\Reals{n\times n}$ for an undirected graph $\mathcal{G}$ is constructed such that each $(i,j)$\textsuperscript{th} entry is $l_{ij}$ and 
 \begin{equation*}
 \label{eq:graph_laplacian}
 	l_{ij} = 
 	\begin{cases}
 		|\Neighs{i}| \, , 	&\text{if} \ i=j \\
 		-1\, ,				&\text{if} \ i\neq j \ \text{and} \ (i,j)\in\mathcal{E} \\
 		0\, ,				&\text{otherwise}
 	\end{cases}
 	\, .
 \end{equation*}
It is well known that for undirected graphs $\mathbf{L}$ has eigenvalues $\lambda_1 = 0 \leq \lambda_2 \leq \hdots \leq \lambda_n$~\cite{olfati2004consensus}. 
The eigenvector associated to $\lambda_1$ is in the set of vectors with equal entries, i.e., $\bm{\lambda}_1 \in \Ind_n = \{ \mathbf{a} \in \Reals{n} \mid \mathbf{a} = \alpha \One{n},\ \alpha\in\Reals{} \}$~\cite{chung1997spectral}, where $\One{n}$ represents the $n$-dimensional column vector of ones. 
This vector is then a left and right eigenvector of $\mathbf{L}$ --- as graph $\mathcal{G}$ is undirected --- and the null space of $\mathbf{L}$ is $\Null(\mathbf{L}) = \Ind_n$. 
Lastly, the full rotation estimates lie in $SO(3)^n = \{ \R_1, \hdots, \R_n \mid \R_i \in SO(3) ,\, \forall i\in\mathcal{V} \}$. 
 


\section{The \texorpdfstring{$SO(3)$}\ \ matrix Lie Group}
\label{sec:lie_group}

We make extensive use of the \textit{Lie group} $SO(3)$ in the analysis of GeoD and therefore we review several relevant properties, as well as introduce a new useful result. Both $SO(3)$ and $SE(3)$ are Lie groups which are simultaneously groups and differentiable manifolds~\cite{barfoot2017state}. Every Lie group has an associated \textit{Lie algebra} which is defined as the \textit{tangent space} of the Lie group at identity element of the group. Lie algebra often take non-trivial forms but have associated vector spaces. 

Recall that $SO(3)$ is the set of valid $3$D rotation matrices. The Lie algebra of $SO(3)$ is $so(3) = \{\mathbf{S}\in\Reals{3\times 3} \mid \mathbf{S}^T = -\mathbf{S}\}$ which is the set of skew-symmetric $3\times 3$ matrices. Define the \textit{vee} operator $(\cdot)^{\vee}: so(3) \rightarrow \Reals{3}$ as the map between the Lie algebra and the associated vector space. The inverse operation is the \textit{hat} operator $(\cdot)^{\wedge} : \Reals{3} \rightarrow so(3)$. Elements of the Lie algebra are mapped back to the group $SO(3)$ using the \textit{exponential map} $\exp(\x^{\wedge}): so(3) \rightarrow SO(3)$, which is given in closed form by the Rodrigues' rotation formula, 
\begin{equation*} \begin{small} \begin{aligned}
	\exp(\x^{\wedge}) &= \Eye{3} \, + \sin(\theta) \left(\frac{\x}{\theta}\right)^{\wedge} + (1-\cos(\theta)) {\left(\frac{\x}{\theta}\right)^{\wedge}}^2
	\, ,
\end{aligned} \end{small} \end{equation*}
where $\x\in\Reals{3}$, $\x^{\wedge} \in so(3)$, and $\theta = \norm{\x}_2$. 
The inverse mapping is the \textit{logarithmic map} $\log(\R): SO(3) \rightarrow so(3)$,
\begin{equation}
\label{eq:matrix_logarithm}
	\log(\R) = 
	\begin{cases}
		\frac{\theta}{2\sin(\theta)}(\R-\R^T)\, ,	& \text{if} \ \theta\neq 0 \\
		0\, ,							&  \text{if} \ \theta=0
	\end{cases}
	\, ,
\end{equation}
where, 
\begin{equation}
\label{eq:angle_axis_theta}
    \theta = \arccos\left(\frac{1}{2}(\tr(\R)-1)\right)
    \, ,
\end{equation} and $\tr(\cdot)$ returns the trace of a matrix. The logarithmic map, or sometimes called the matrix logarithm, is defined for $|\theta| < \pi$ to ensure the mapping is bijective. 
The \textit{angle-axis} representation of the rotation matrix $\R$ is defined by the corresponding element of the vector space $\x = \theta\u = \log(\R)^{\vee}$, where $\theta=\norm{\x}_2 = \norm{\log(\R)^{\vee}}_2$ is the magnitude of the rotation about a unit axis $\u = \frac{\x}{\norm{\x}_2}$. This is the same $\theta$ defined in~\eqref{eq:angle_axis_theta} and corresponds to the geodesic distance along the manifold from the identity element to $\R$. 
Conversely, we can write the rotation matrix as a function of the angle-axis, i.e., $\R = \exp(\x^{\wedge})$.

In this paper, we are primarily interested in expressions that are functions of Lie algebra (see~\eqref{eq:rotation_estimation_system}) and thus state some general properties of the matrix logarithm on real-valued matrices. 
From properties of the matrix exponential, we know that $\exp(\mathbf{B}\A\mathbf{B}^{-1}) = \mathbf{B}\exp(\A)\mathbf{B}^{-1}$ for any $\A\in\Reals{d\times d}$ and invertible $\mathbf{B}\in \Reals{d\times d}$. 
If the logarithmic map of this generic $\A$ exists such that $\log(\exp(\A)) = \exp(\log(\A)) = \A$, then we can also write $\log(\mathbf{B} \A \mathbf{B}^{-1}) = \mathbf{B} \log(\A) \mathbf{B}^{-1}$. 
For rotation matrices $\R, \mathbf{Q} \in SO(d)$, the equivalent expression is, 
\begin{equation}
\label{eq:log_products_result}
	\log(\mathbf{Q} \R \mathbf{Q}^T) = \mathbf{Q} \log(\R) \mathbf{Q}^T
	\, .
\end{equation}
Note that in general $\log(\R^T\mathbf{Q}) \neq \log(\R) - \log(\mathbf{Q})$, however this statement is indeed true for rotations in $SO(2)$~\cite{thunberg2011distributed}. 
Additionally, $\log(\mathbf{R}^{T}) = \log(\mathbf{R})^{T} = -\log(\mathbf{R})$. 
Finally, we state the following lemma for several well-known results involving skew-symmetric matrices~\cite{lee2010geometric, lee2010control}. 
\begin{lemma}
\label{lem:skew_symmetric}
	It follows that for $\mathbf{x}, \mathbf{y} \in\Reals{3}$, and $\mathbf{A} \in \Reals{3\times 3}$ that, 
	\begin{equation}
	\label{eq:dot_product_trace}
		\mathbf{x}^T\mathbf{y} = \frac{1}{2} \tr \left( {\mathbf{x}^{\wedge}}^T\mathbf{y}^{\wedge} \right)
		\, ,
	\end{equation}
	\begin{equation}
	\label{eq:trace_rotation_skew_product}
		\tr(\mathbf{A} \mathbf{x}^{\wedge}) = -\mathbf{x}^T(\mathbf{A} - \mathbf{A}^T)^{\vee}
		\, .
	\end{equation}
\end{lemma}

The rate of change of a rotation matrix is well known given a vector of rotation rates $\bm\omega \in\Reals{3}$: $\dot\R = \R \bm\omega^{\wedge}$ and $\dot\R = \bm\omega^{\wedge}\R$ for a body-frame $\bm\omega$ and global-frame $\bm\omega$, respectively. Unfortunately, finding a closed-form expression for the time rate of change of $\log(\R)$ is challenging because it requires differentiating an infinite series representation of the matrix logarithm (See 3.3 of~\cite{dieci1996computational}) or, for $SO(3)$, differentiating~\eqref{eq:matrix_logarithm} with respect to time --- leading to a cumbersome expression in both cases. Below we state a theorem that will be central to the Lyapunov analysis of our algorithm in Section~\ref{sec:analysis_rotation_matrix}. The theorem allows us to compactly express the time derivative of the squared geodesic distance, without explicitly requiring the time derivative of the matrix logarithm itself. 
\begin{theorem}
\label{thm:time_derivative_dot_matrix_logarithm}
	For a rotation matrix $\R\in SO(3)$, if
	$\norm{\log(\R)^{\vee}}_2 < \pi$, then, 
	\begin{equation}
		\frac{d}{dt} \frac{1}{2} \|\log(\R)^{\vee}\|_2^2 = {\log(\R)^{\vee}}^T (\R^T\dot\R)^{\vee}
		\, .
	\end{equation}
\end{theorem}
\begin{proof}
See Appendix~\ref{app:lie_group}. 
\end{proof}

\section{Analysis}
\label{sec:analysis}

In this section we prove that the dynamics described by the system~\eqref{eq:estimation_system} will converge to an equilibrium if the noisy measurements $(\tilde\t_{ij}, \tilde\R_{ij})$ are \textit{consistent}. 
Before we begin the analysis, we introduce our taxonomy of measurement consistency with three definitions. 
\begin{definition}[Global consistency]
\label{def:global_consistency}
The measurements are \textit{globally consistent} if $\tilde{\R}_{ij} = \tilde{\R}_{i \ell} \tilde{\R}_{\ell j}$ and $\tilde\t_{ij} = \tilde\t_{i\ell} + \tilde\R_{i\ell}\tilde\t_{\ell j}$, for all $i,\ell,j \in\mathcal{V}$. 
In other words, the product of relative pose measurements along any cycle in the graph must return the identity element. 
\end{definition}
\begin{definition}[Pairwise consistency]
The measurements are \textit{pairwise consistent} if $\tilde{\R}_{ij} = \tilde{\R}_{ji}^T$ and $\tilde\t_{ij} = \tilde\R_{ij}\tilde\t_{ji}$, $\forall \ (i,j) \in\mathcal{E}$. 
\end{definition}
\begin{definition}[Minimal consistency]
\label{def:min_consistency}
The measurements are \textit{minimally consistent} if $\sum_{i\in\mathcal{V}}\sum_{j\in\Neighs{i}} \log(\tilde\R_{ij})^{\vee} = \Zero{3}$ and $\sum_{i\in\mathcal{V}}\sum_{j\in\Neighs{i}} \tilde\t_{ij} = \Zero{3}$. 
\end{definition}

Global consistency is the most restrictive form of consistency. It often appears in formation control literature where the desired formations are assumed to be provided to the system as inputs and represent a valid, desired formation~\cite{montijano2016vision}. Global consistency is almost always violated with real-world noisy relative measurements. 
Pairwise consistency is less restrictive than global consistency because it only requires consistency among each undirected edge as opposed to all cycles in the graph. This condition is easy to satisfy in practice with a distributed system because the raw measurements can be shared and averaged along each edge in a single step (see Section~\ref{sec:results} for more details). 
Minimal consistency is physically the least restrictive definition as it encompasses both global and pairwise consistency. However, achieving minimal consistency within a distributed network is more complicated than pairwise consistency. 

In Section~\ref{sec:analysis_rotation_matrix}, we first prove that \textit{pairwise consistency} is a sufficient condition for convergence of the rotation-only estimation system~\eqref{eq:rotation_estimation_system} using LaSalle's Invariance Principle. 
We then show in Section~\ref{sec:analysis_coupled} that \textit{minimal consistency} is a sufficient condition for the equilibria of the coupled translation system~\eqref{eq:translation_estimation_system} to be input-to-state stable. Note that this system is a function of the time-varying rotation estimates. Finally, we show that the coupled system, with rotations and translations simultaneously evolving, converges to an equilibrium. Note that equilibria of (\ref{eq:estimation_system}) are local minima of the pose graph objective function in Problem~\ref{prob:main}.

\subsection{Equilibria of the \texorpdfstring{$SO(3)$}\ \ Rotation Matrix Estimate}
\label{sec:analysis_rotation_matrix}

We now show that pairwise consistent relative rotation measurements, i.e., $\tilde\R_{ij} = \tilde\R_{ji}^T,\ \forall (i,j)\in\mathcal{E}$, implies the convergence of the rotation estimate~\eqref{eq:rotation_estimation_system}, i.e., $\dot\R_{i} = \Zero{3\times 3},\ \forall i\in\mathcal{V}$. 
For the analysis, let $\R_{ij} = \R_i^T \R_j$ represent the rotation product and $\bm\omega_{ij}^{\wedge}$ the \textit{relative control components}, 
\begin{equation} \begin{aligned}
\label{eq:relative_control_component}
	\bm{\omega}_{ij}^{\wedge} &= \log(\R_i^T\R_j \tilde\R_{ij}^T) = \log(\R_{ij}\tilde\R_{ij}^T) \\
	\bm{\omega}_{ji}^{\wedge} &= \log(\R_j^T\R_i \tilde\R_{ji}^T) = \log(\R_{ij}^T\tilde\R_{ij})
	\, .
\end{aligned} \end{equation}
so that $\bm\omega_i^{\wedge} = \sum_{j\in\Neighs{i}} \bm{\omega}_{ij}^{\wedge}$; noting that $\bm\omega_{ij}^{\wedge} \neq -\bm\omega_{ji}^{\wedge}$. 

We introduce the following Lemmas to assist in the rotation estimate convergence result in Theorem~\ref{thm:geodesic} that use properties of the matrix logarithm and the relative control components from~\eqref{eq:relative_control_component}. For clarity, the proofs of these helper Lemmas are shown in Appendix~\ref{app:analysis_rotation_matrix}. 

\begin{lemma}
\label{lem:rotated_control}
	For any $\mathbf{Q} \in \{ \R_{ij}, \tilde\R_{ij} \}$, if the rotations measurements are pairwise consistent, then, 
\begin{equation}
	\mathbf{Q}^T\bm\omega_{ij}^{\wedge}\mathbf{Q} = -\bm\omega_{ji}^{\wedge}
	\, .
\end{equation}
\end{lemma}

\begin{lemma}
\label{lem:dot_product_equivalence_1}
	For a pair of relative control components $\bm\omega_{ij}$ and $\bm\omega_{ik}$ as well as rotation matrix $\tilde\R_{ij}$, the following are equivalent, 
	\begin{equation}
		\bm\omega_{ij}^T \left( \tilde\R_{ij} \bm\omega_{ik}^{\wedge} \tilde\R_{ij}^T \right)^{\vee} = -\bm\omega_{ji}^T\bm\omega_{ik}
		\, .
	\end{equation}
\end{lemma}

\begin{lemma}
\label{lem:dot_product_equivalence_2}
	For a pair of relative control components $\bm\omega_{ij}$ and $\bm\omega_{ik}$ as well as rotation matrices $\R_{ij}$ and $\tilde\R_{ij}$, the following dot products are equivalent, 
	\begin{equation}
		\bm\omega_{ij}^T \left( \tilde\R_{ij}\R_{ij}^T \bm\omega_{ik}^{\wedge} \R_{ij}\tilde\R_{ij}^T \right)^{\vee} = \bm\omega_{ij}^T\bm\omega_{ik}
		\, .
	\end{equation}
\end{lemma}

\begin{lemma}
\label{lem:relative_control_product}
	For $\bar\R = \R_{ij}\tilde\R_{ij}^T$, the product $\bar\R^T \dot{\bar\R}$ is, 
	\begin{equation}
		 \bar\R^T \dot{\bar\R} = 
		- \tilde\R_{ij}\R_{ij}^T \bm\omega_i^{\wedge} \R_{ij}\tilde\R_{ij}^T + \tilde\R_{ij}\bm\omega_j^{\wedge}\tilde\R_{ij}^T \in so(3)
		\, .
	\end{equation}
\end{lemma}

\begin{theorem}[Rotation estimation convergence]
\label{thm:geodesic}
	If, $\forall (i,j) \in \mathcal{E}$, the rotation measurements $\tilde\R_{ij}$ and $\tilde\R_{ji}$ are pairwise consistent and the initial rotation estimates $\R_i$ lie in $\mathcal{B}_{\frac{\pi}{2}-\varepsilon} = \{ \R_1, \hdots, \R_n \mid \norm{\log(\R_i^T\R_j\tilde\R_{ij}^T)}_2 \leq (\frac{\pi}{2}-\varepsilon) ,\, \forall (i, j)\in\mathcal{E} \} \subset SO(3)^n$, where $\varepsilon\in\mathbb{R}_{+}$ is an arbitrarily small nonnegative number, then the estimates evolving according to~\eqref{eq:rotation_estimation_system} will converge to the set of equilibria $\mathbb{E}_{\R} = \{ \R_1, \hdots, \R_n \mid \dot\R_i = \R_i\bm\omega_i^{\wedge} = \R_i \sum_{j\in\Neighs{i}} \bm\omega_{ij}^{\wedge} = \Zero{3 \times 3} ,\, \forall i\in\mathcal{V}\}$. 
\end{theorem}
\begin{proof}

First, the set $\mathcal{B}_{\frac{\pi}{2}-\varepsilon}$ is positively invariant under~\eqref{eq:rotation_estimation_system} according on the analysis in~\cite{thunberg2014distributed} (Theorem 10) which leverages properties from~\cite{moakher2002means}. Our controller differs from~\cite{thunberg2014distributed} due to the rotation measurement offset terms, however the analysis that compares the gradient and the covariant derivative on $SO(3)$ holds for our controller as well. 

Define a Lyapunov-like function $V : SO(3)^n \to \Reals{}$ that is a function of the agent's rotation matrices:
\begin{equation}
\label{eq:lyapunov_function}
    V (\R_1, \hdots, \R_n) = \sum_{i\in\mathcal{V}} \sum_{j\in\Neighs{i}} \frac{1}{2}\norm{\log( \R_i^T \R_j \tilde\R_{ij}^T)^{\vee}}_2^2.
\end{equation}
This is a sum of squared relative control components from~\eqref{eq:relative_control_component} over all directed edges in the graph. 
The time derivative of $V$ is written using the time derivative of the squared geodesic distance from Theorem~\ref{thm:time_derivative_dot_matrix_logarithm} as well as the result from Lemma~\ref{lem:relative_control_product}, 
\begin{equation} \begin{small} \begin{aligned}
	\dot{V}
	&= \sum_{i\in\mathcal{V}} \sum_{j\in\Neighs{i}} \frac{1}{2}\frac{d}{dt} \norm{\log( \R_{ij}\tilde\R_{ij}^T)^{\vee}}_2^2 \\
	&= \sum_{i\in\mathcal{V}} \sum_{j\in\Neighs{i}} {\log( \R_{ij}\tilde\R_{ij}^T)^{\vee}}^T \left( (\R_{ij}\tilde\R_{ij}^T)^T\frac{d}{dt}(\R_{ij}\tilde\R_{ij}^T) \right)^{\vee} \\
	&= \sum_{i\in\mathcal{V}} \sum_{j\in\Neighs{i}} \bm\omega_{ij}^T \left( - \tilde\R_{ij}\R_{ij}^T \bm\omega_i^{\wedge} \R_{ij}\tilde\R_{ij}^T + \tilde\R_{ij}\bm\omega_j^{\wedge}\tilde\R_{ij}^T \right)^{\vee}
	\, .
\end{aligned} \end{small} \end{equation}
Each control term $\bm\omega_i^{\wedge}$ is now replaced by the definition of the rotational control input in~\eqref{eq:estimation_control_rotation}, 
\begin{equation} \begin{small} \begin{aligned}
	\dot{V}
	&= \sum_{i\in\mathcal{V}} \sum_{j\in\Neighs{i}} \bm\omega_{ij}^T \Bigg( - \tilde\R_{ij}\R_{ij}^T \Bigg(\sum_{k\in\Neighs{i}}\bm\omega_{ik}^{\wedge} \Bigg) \R_{ij}\tilde\R_{ij}^T \\
		&\qquad\qquad\qquad\quad + \tilde\R_{ij} \Bigg(\sum_{l\in\Neighs{j}}\bm\omega_{jl}^{\wedge} \Bigg) \tilde\R_{ij}^T \Bigg)^{\vee} \\
	&= \sum_{i\in\mathcal{V}} \sum_{j\in\Neighs{i}} \Bigg( \sum_{k\in\Neighs{i}} - \bm\omega_{ij}^T \left( \tilde\R_{ij}\R_{ij}^T \bm\omega_{ik}^{\wedge} \R_{ij}\tilde\R_{ij}^T \right)^{\vee} \\ 
		&\qquad\qquad\qquad\quad + \sum_{l\in\Neighs{j}} \bm\omega_{ij}^T \left( \tilde\R_{ij} \bm\omega_{jl}^{\wedge} \tilde\R_{ij}^T \right)^{\vee} \Bigg)
	\, .
\end{aligned} \end{small} \end{equation}
And finally, we use the dot product equivalence results from Lemmas~\ref{lem:dot_product_equivalence_1} and~\ref{lem:dot_product_equivalence_2} to arrive at the following, 
\begin{equation} \begin{small} \begin{aligned}
\label{eq:Vdot_sums_before_doublecount}
	\dot{V}
	&= \sum_{i\in\mathcal{V}} \sum_{j\in\Neighs{i}} \Bigg( - \sum_{k\in\Neighs{i}} \bm\omega_{ij}^T \bm\omega_{ik} - \sum_{l\in\Neighs{j}} \bm\omega_{ji}^T \bm\omega_{jl} \Bigg)
	\, .
\end{aligned} \end{small} \end{equation}
Notice that for each directed edge $(i, j)$, there is a sum over neighbors of $i$ including $\bm\omega_{ij}$ and a sum over all neighbors of node $j$ containing $\bm\omega_{ji}$. Since we are considering a directed edge, these neighbors will be counted again when we consider the directed edge $(j, i)$ in the opposite direction --- see the illustration of each edge's sums in Fig.~\ref{fig:geodesic_proof}. Thus, we can include these double counted edges for each directed edge, 
\begin{equation} \begin{aligned}
	\dot{V}
	&= \sum_{i = 1}^n \sum_{j \in \Neighs{i}} \left( - 2\sum_{k\in\Neighs{i}} \bm\omega_{ij}^T \bm\omega_{ik} \right) \\
	&= -2 \sum_{i = 1}^n \sum_{j \in \Neighs{i}} \sum_{k\in\Neighs{i}} \bm\omega_{ij}^T \bm\omega_{ik}
	\, .
\end{aligned} \end{equation}
Finally, we rearrange the two sums over neighbors of $i$ into the product of the same sums over neighbors of $i$,
\begin{equation} \begin{aligned}
	\dot{V}
	&= -2 \sum_{i = 1}^n \left( \sum_{j \in \Neighs{i}} \bm\omega_{ij}^T \right) \left(\sum_{j\in\Neighs{i}} \bm\omega_{ij} \right) \\
	&= -2 \sum_{i = 1}^n \bm\omega_i^T \bm\omega_i
	\, .
\end{aligned} \end{equation}
Therefore the set of rotation estimates where $\dot V = 0$ corresponds to the estimates where the control inputs $\bm{\omega}_i$ are all zero, i.e., the rotation estimate equilibria set $\mathbb{E}_{\R}$. 
\end{proof}

\begin{figure}
\centering
\subfigure[\label{fig:geodesic_proof_ij}]{
	\begin{tikzpicture}[scale=1]
	\pgfmathsetmacro{\offset}{0.18}
	\draw (0,0) node{$i$};
	\draw (1,0) node{$j$};
	\draw (0,-1) node{$k$};
	\draw (1,-1) node{$l$};
	\draw[draw=red,-stealth,thin,->] (\offset,0.08) -- (1-\offset,0.08);
	\draw[-stealth,thin,<-] (\offset,-0.08) -- (1-\offset,-0.08);
	\draw[-stealth,thin,->] (0,-\offset) -- (0,-1+\offset);
	\draw[-stealth,thin,->] (1,-\offset) -- (1,-1+\offset);
	\end{tikzpicture}
}
\subfigure[\label{fig:geodesic_proof_ji}]{
	\begin{tikzpicture}[scale=1]
	\pgfmathsetmacro{\offset}{0.18}
	\draw (0,0) node{$i$};
	\draw (1,0) node{$j$};
	\draw (0,-1) node{$k$};
	\draw (1,-1) node{$l$};
	\draw[-stealth,thin,->] (\offset,0.08) -- (1-\offset,0.08);
	\draw[draw=red,-stealth,thin,<-] (\offset,-0.08) -- (1-\offset,-0.08);
	\draw[-stealth,thin,->] (0,-\offset) -- (0,-1+\offset);
	\draw[-stealth,thin,->] (1,-\offset) -- (1,-1+\offset);
	\end{tikzpicture}
}
\caption{Example of the directed edges represented in the sums of~\eqref{eq:Vdot_sums_before_doublecount}. Fig.~\ref{fig:geodesic_proof_ij} represents the directed edge $(i,j)$ and the neighbors of $i$ and $j$. For this edge, the neighbors of $i$ show up in the product with $\bm\omega_{ij}$ (first component of sum in~\eqref{eq:Vdot_sums_before_doublecount}). Fig.~\ref{fig:geodesic_proof_ji} represents the directed edge $(j,i)$ and the same neighbors as before. Similarly, the neighbors of $i$ show up in the product with $\bm\omega_{ji}$ (now the second component of sum in~\eqref{eq:Vdot_sums_before_doublecount}). Therefore, when considering one directed edge $(i,j)$, we can instead count the neighbors of $i$ twice and leave the neighbors of $j$ for the opposite edge. }
\label{fig:geodesic_proof}
\end{figure}
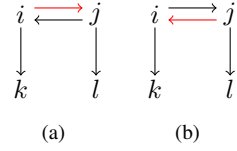


\subsection{Equilibria of the Coupled \texorpdfstring{$SE(3)$}\ \ Translation Estimate}
\label{sec:analysis_coupled}

The coupled translation estimation given by system~\eqref{eq:translation_estimation_system}, 
\begin{equation}
\label{eq:translation_estimation_system_analysis}
	\dot\t_i = \sum_{j\in\Neighs{i}} \t_j - \t_i - \R_i \tilde\t_{ij}
	\, ,
\end{equation}
where the $\R_i$ terms are the time-varying rotation estimates that were analyzed in the previous section. 
We consider the stacked version of system~\eqref{eq:translation_estimation_system_analysis} and define vector $\t = [\t_1^T,\hdots,\t_n^T]^T \in \Reals{3n}$ such that the stacked dynamics of translation estimate are, 
\begin{equation}
\label{eq:translation_dynamics_disturbance}
	\dot\t = -(\mathbf{L}\otimes\Eye{3})\t - \bm{\delta}(t)
	\, ,
\end{equation}
where $\mathbf{L} \in\Reals{n\times n}$ is the graph Laplacian, $\otimes$ represents the Kronecker product required for $3$D estimates, and $\bm{\delta}\in \Reals{3n}$ is the vector of desired offsets where each component is $\bm\delta_{i}(t) = \sum_{j\in\Neighs{i}} \R_i \tilde{\t}_{ij} \in\Reals{3}$. Note that the desired offsets are functions of the rotation estimates. The equlibria set of this joint system is defined as $\mathbb{E}_{\t} = \{\t \in \Reals{3n} \mid \dot{\t}=\Zero{3n} \}$ such that a $\t \in\mathbb{E}_{\t}$ is an equilibrium point of~\eqref{eq:translation_estimation_system}. 

The analysis of this coupled estimation system follows from the coupled $3$D estimation analysis: Corollary 8 from~\cite{cristofalo2019consensus}. 
\begin{corollary}[Translation estimation input-to-state stability]
\label{cor:translation_rotation_equilibria}
	If the distributed pose estimate system~\eqref{eq:estimation_dynamics} is controlled by~\eqref{eq:estimation_control}, the rotation measurements are minimally consistent, and the coupled rotation-translation offsets obey $\lim_{t \rightarrow \infty} \bm\delta(t) = \bm\delta_{\infty}$ with $\bm\delta_{\infty}$ minimally consistent, then, 
    \begin{equation}
    	\lim_{t \rightarrow \infty} \t = \t_{eq}
    	\, ,
    \end{equation}
    where $\t_{eq} \in \mathbb{E}_{\t}$ is an equilibrium point of~\eqref{eq:translation_estimation_system}. 
\end{corollary}
\begin{proof}
	    The proof directly follows Corollary 8 of~\cite{cristofalo2019consensus}. In brief, the rotation estimates converge (Theorem~\ref{thm:geodesic}) to constant values, thus the translation offsets $\bm\delta (t)$ can be set as minimally consistent (e.g., via pairwise averaging). Thus, the forced component of the translation estimation error dynamics vanishes asymptotically (Proposition 7 of~\cite{cristofalo2019consensus}) and, from Lemma 4.6 of~\cite{khalil2015nonlinear}, the system is input-to-state stable. 
\end{proof}

The result in Corollary~\ref{cor:translation_rotation_equilibria} tells us that the translation and rotation estimation can be run concurrently. Previous distributed estimation works often consider these in separate, sequential algorithms, however this is not necessary with our method. 
Additionally, Corollary~\ref{cor:translation_rotation_equilibria} requires that the coupled rotation-translation offsets tend towards a minimally consistent vector. 
How exactly that consistency is achieved is a design element of the algorithm and may be accomplished in many ways. We propose a number of options here and relate them to other algorithms in literature. 

Since rotation is independent of translation, an intuitive technique is to average the initial pairwise rotation measurements between connected neighbors. The rotation measurements may be initialized at the beginning of the algorithm for each edge pair $(i,j) \in\mathcal{E}$ by, 
\begin{equation}
\label{eq:pairwise_rot}
	\tilde\R_{ij}^{\prime} = \tilde\R_{ij}\exp\left(\frac{1}{2}\log(\tilde\R_{ij}^T\tilde\R_{ji}^T)\right)
	\, ,
\end{equation}
such that $\tilde\R_{ij}^{\prime}\tilde\R_{ji}^{\prime} = \Eye{3}$. This operation is not expensive as it involves one communication round between agent $i$ and its immediate neighbors. Further, if the original graph is directed, it can be transformed into an undirected graph by inverting measurements for the opposite directions which will inherently constitute a pairwise consistent undirected edge. 

For translation, one must consider the time-varying rotation estimate. One approach is to average the coupled translation terms between each update loop using, 
\begin{equation}
\label{eq:pairwise_trans}
	\tilde\t_{ij}^{\prime} = \frac{1}{2}(\tilde\t_{ij} - \R_{ij}^{T}\tilde\t_{ji})
	\, .
\end{equation}
This operation only requires a similar one-hop communication at the beginning, however agent $i$ can perform this operation without additional communication at all steps in the future. This is the technique used in Section~\ref{sec:results}. 

In the algorithm by Tron et al.~\cite{tron2014distributed}, the estimates are updated using the gradient of an objective similar to the one in Problem~\ref{prob:main}. 
Notably, these gradient-based updates naturally enforce pairwise consistency of the noisy measurements and thus satisfy our requirement for stability. 
Inspired by their method, this real-time consistency can be incorporated into a modified control law for the estimation system~\eqref{eq:estimation_dynamics}: 
\begin{equation}
\label{eq:translation_control_online}
	\bm\nu_i = \sum_{j\in\Neighs{i}} (\t_j - \t_i) + \frac{1}{2} \left(\R_j \tilde\t_{ji} - \R_i \tilde\t_{ij} \right)
	\, .
\end{equation}
This type of pairwise averaging also appears in~\cite{thunberg2017distributed} (Algorithm 1) 
despite the different form of rotation. 
We also relate our findings to formation control literature such as~\cite{montijano2016vision}, where they require global consistency on the desired poses for the proof of convergence of the coupled translation and rotation formation. The global consistency requirement is significantly more restrictive than our pairwise and minimal consistency. However, our method would also converge given globally consistent measurements if they were somehow provided. 
\section{Empirical Results and Performance Comparisons}
\label{sec:results}

We present results from deploying GeoD\footnote{\href{https://github.com/ericcristofalo/GeoD}{github.com/ericcristofalo/GeoD}} on a series of simulated and experimental pose graph datasets in Section~\ref{sec:results_simulation}. 
We also show results from a distributed multi-robot SLAM and $3$D reconstruction scenario in Section~\ref{sec:results_experiment}, in which seven quadrotor robots share their camera images with their neighbors to estimate their relative poses along their trajectories. We run GeoD using the relative poses collected from these robots to estimate the group's poses while simultaneously building a $3$D reconstruction of an unknown scene. 

For all  datasets, we compare GeoD's performance with the state-of-the-art centralized $SE(3)$ pose synchronization algorithm, \textit{SE-Sync}~\cite{rosen2016se}, and a Distributed Gauss-Seidel algorithm, \textit{DGS}~\cite{choudhary2017distributed}. 
Both methods solve a convex relaxation of a pose graph optimization problem that is similar to Problem~\ref{prob:main}, except using the chordal distance instead of the geodesic distance as we do. 
SE-Sync solves a semi-definite relaxation version of chordal problem while DGS solves a quadratic relaxation of the chordal problem in a distributed fashion. 
We implement the variant of DSG that uses Jacobi Over-Relaxation (JOR)~\cite{choudhary2017distributed}. 
SE-Sync provides a certificate of optimality if the final solution of the relaxed problem yields the same result as the original problem. In this way, we can compare our distributed method to the centralized, chordal global minimum. In this section, we denote this global minimum using the chordal distance with a *. We reiterate that this global minimum is for the chordal objective and not the geodesic objective shown in Problem~\ref{prob:main}. We report both objective values in this section for completeness. 

In practice, we implement a discretized version of GeoD~\eqref{eq:estimation_system} with a time step of $0.05\ \mathrm{seconds}$. 
The only requirement to use our method is the pairwise averaging condition. We enforce this by averaging the rotation measurement between neighboring robots once using~\eqref{eq:pairwise_rot} before running GeoD, and then by enforcing~\eqref{eq:pairwise_trans} upon each time step. 
Convergence for the distributed methods is determined by the stopping condition for the difference between consecutive objectives values decreasing below $10^{-2}$, an order of magnitude lower than the condition used in~\cite{choudhary2017distributed}. 

Although this paper provides convergence results using the distributed estimation scheme, the approach is not guaranteed to reach the global minimum due to the gradient-like nature of consensus methods. 
Therefore the initialization will impact the final solution. 
The initialization quality is not the focus of this paper, but an excellent discussion on initialization is provided in~\cite{carlone2015initialization}. 
We initialize the pose estimates using simulated \textit{GPS-like} measurements, assuming a robot has access to GPS and an inertial Measurement Unit (IMU), to provide a translation and rotation estimate. 
If such sensors are not available, we can initialize the poses by computing a minimum \textit{spanning tree} through the graph that represents the chain of relative pose transformations from am arbitrary root node. There exists standard distributed algorithms to compute such a spanning tree that can be computed in a time on the order of the number of robots in the network~\cite{awerbuch1987optimal}. Random or identity initialization will also work with GeoD. 


\subsection{Simulation Results}
\label{sec:results_simulation}

We demonstrate our distributed method with two sets of numerical simulations that illustrate
\begin{enumerate}
	\item the performance compared to our previous consensus-based distributed algorithm~\cite{cristofalo2019consensus} that uses the angle-axis rotation representation (Section~\ref{sec:results_distributed}), and 
	\item the performance compared to state-of-the-art centralized~\cite{rosen2016se} and distributed~\cite{choudhary2017distributed} methods that utilize chordal relaxations (Section~\ref{sec:results_comparison}). 
\end{enumerate}

We consider seven network topologies: the \textit{random}, \textit{circular}, \textit{grid}, \textit{sphere}, \textit{experiment}, \textit{parking-garage}, and \textit{cubicle} (Fig.~\ref{fig:results_networks}). 
The \textit{random} network's ground truth poses are incrementally generated by uniformly random sampling translations within a communication ball around a randomly selected node such that the graph is guaranteed to be connected. The ground truth rotation matrices are uniformly randomly sampled~\cite{arvo1992fast}. The \textit{circular}, \textit{grid}, and \textit{sphere} networks have prescribed translations and topologies while the rotations are uniformly randomly sampled as well. We generate the noisy pairwise relative measurements for the first four networks by adding aggressive noise to ground truth relative poses. Translation noise is generated by a Gaussian distribution $\mathcal{N}(\Zero{3}, \tau^2 \Eye{3})$, where $\tau = 0.5 \ \mathrm{meters}$, and the rotation noise is sampled in the vector space of $SO(3)$ and mapped back to the group using the exponential map: $\bm{\nu} \sim \mathcal{N}(\Zero{3}, \kappa^2 \Eye{3})$ and $\tilde{\R}_{ij} = \R_{i}^T\R_{j} \exp(\bm{\nu}^{\wedge})$, where $\kappa = 0.524 \ \mathrm{radians}$ or $30 \ \mathrm{degrees}$. 
We evaluate the methods on the \textit{experiment} data from Section~\ref{sec:results_experiment} as well for completeness. 
The \textit{parking-garage} and \textit{cubicle} networks are part of a SLAM benchmark  dataset~\cite{rosen2016se} that includes the noisy measurements and no ground truth positioning. 
These datasets illustrate the ability for GeoD to scale effective to very large pose graphs. 

We initialize the estimates using the GPS-like initialization technique. 
The initial translation estimate is generated by corrupting the ground truth translation with additive zero-mean Guassian noise sampled from same distribution as the relative measurement noise defined above. 
The initial rotation estimate is generated in a similar manner to the relative rotation measurements using the same distribution: $\hat{\R}_{i} = \R_{i} \exp(\mathbf{\bm\nu}^{\wedge})$. 
These lead to exaggerated GPS-like initial conditions that may be much more accurate in real robotics scenarios. 


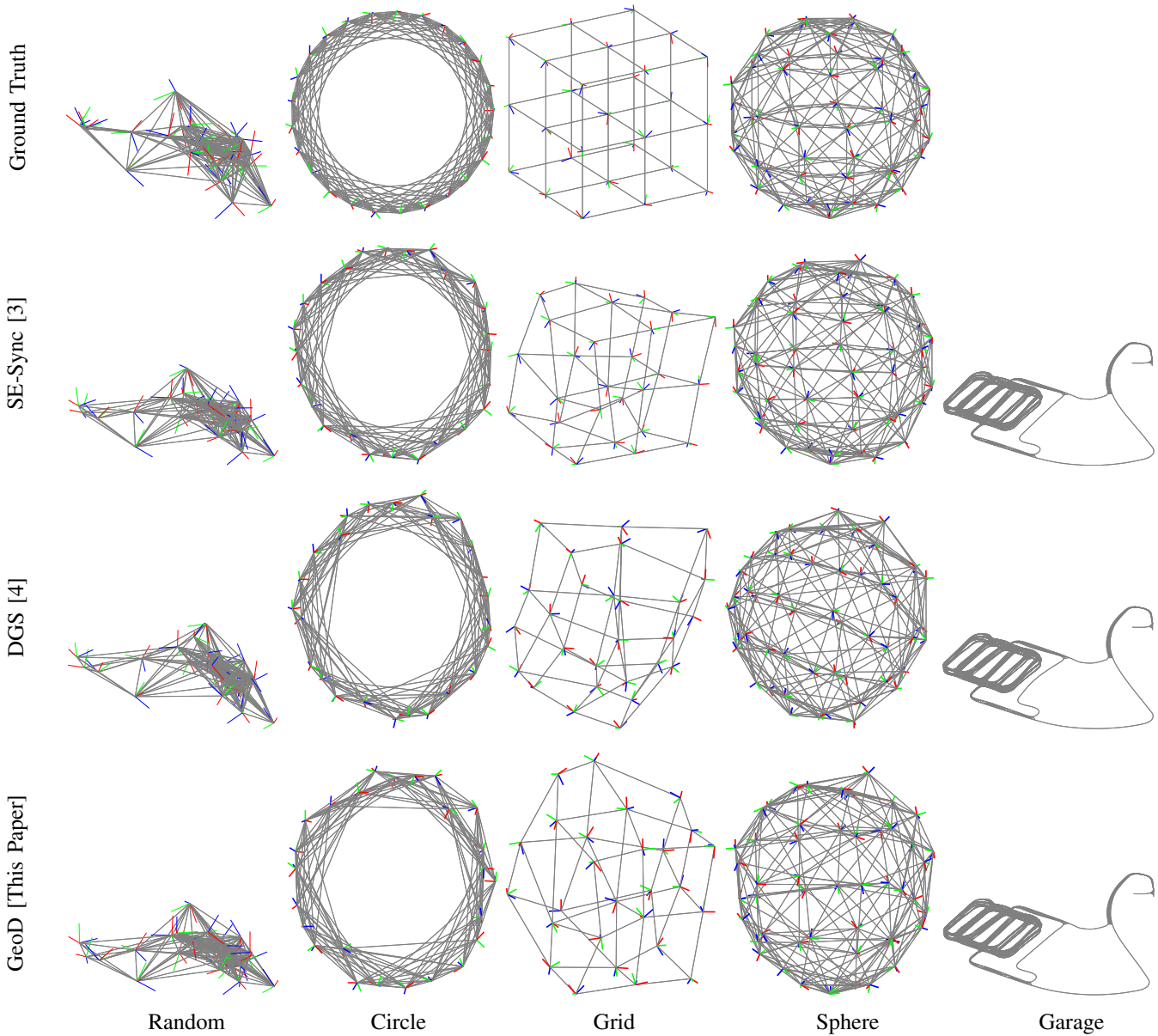
\begin{figure*} [t]
\centering
    \subfigure{\raisebox{15mm}{\rotatebox[origin=t]{90}{Ground Truth}}\hspace{5mm}}
	\subfigure{\includegraphics[width=0.18\linewidth]{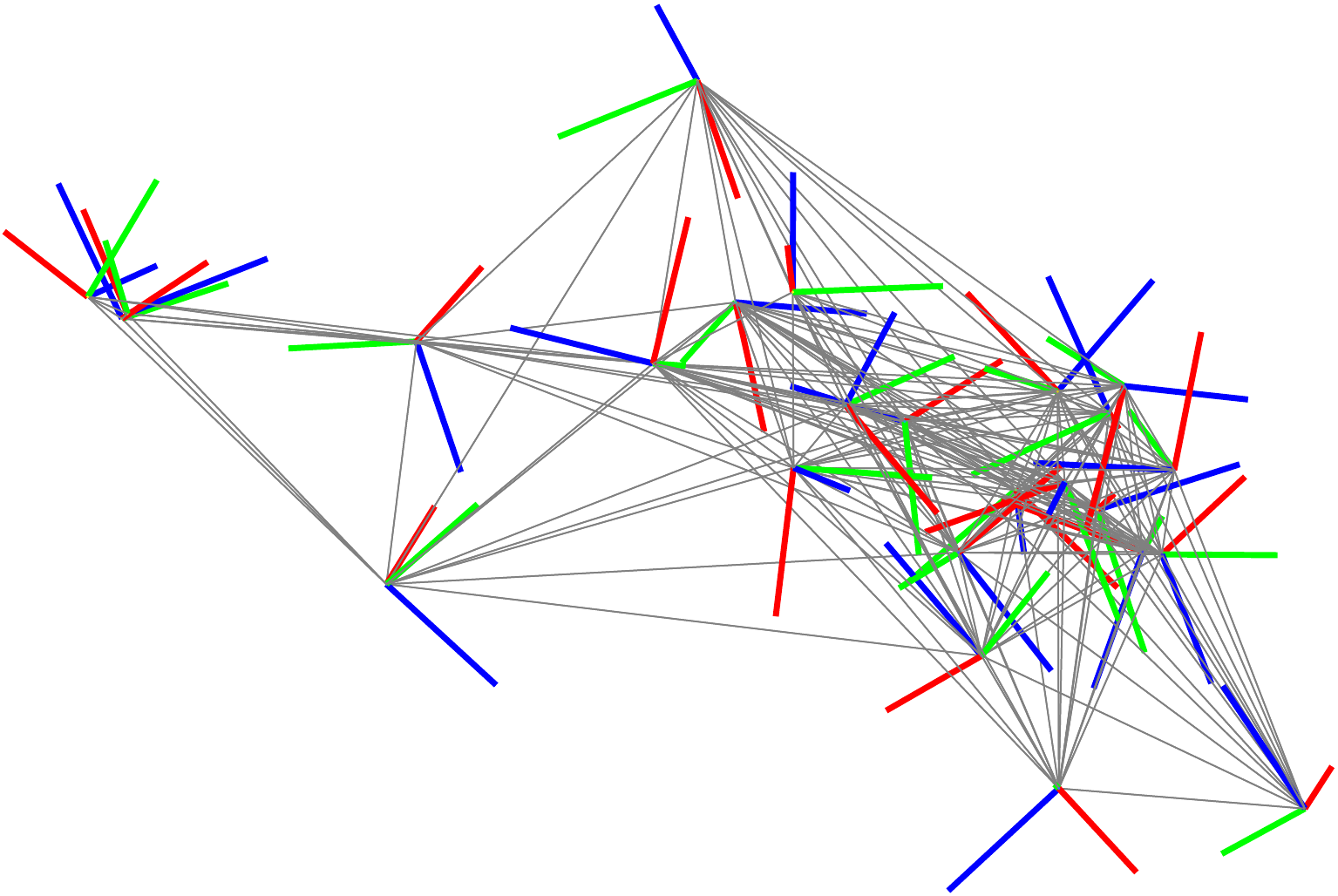}}
	\subfigure{\includegraphics[width=0.18\linewidth]{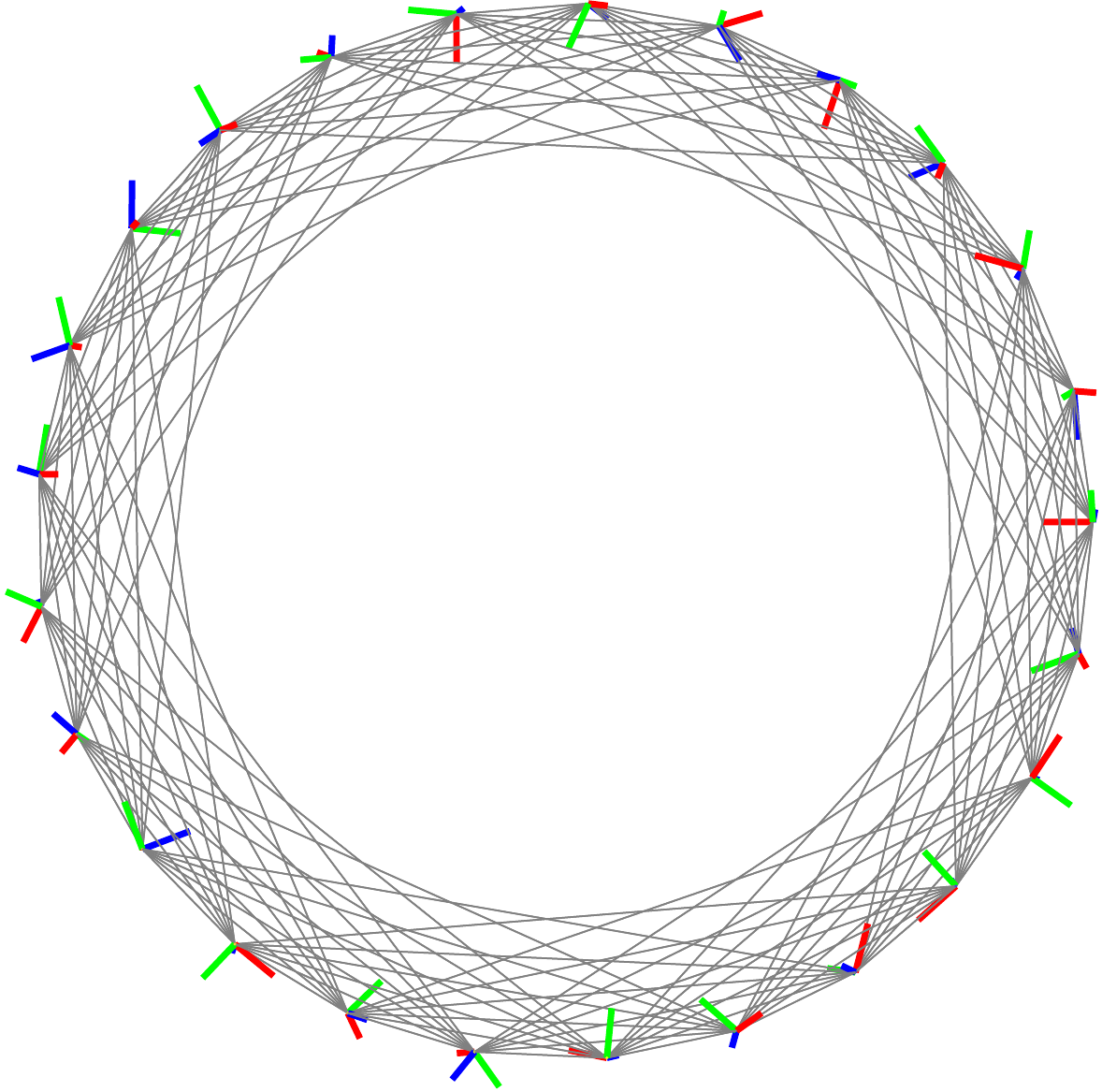}}
	\subfigure{\includegraphics[width=0.18\linewidth]{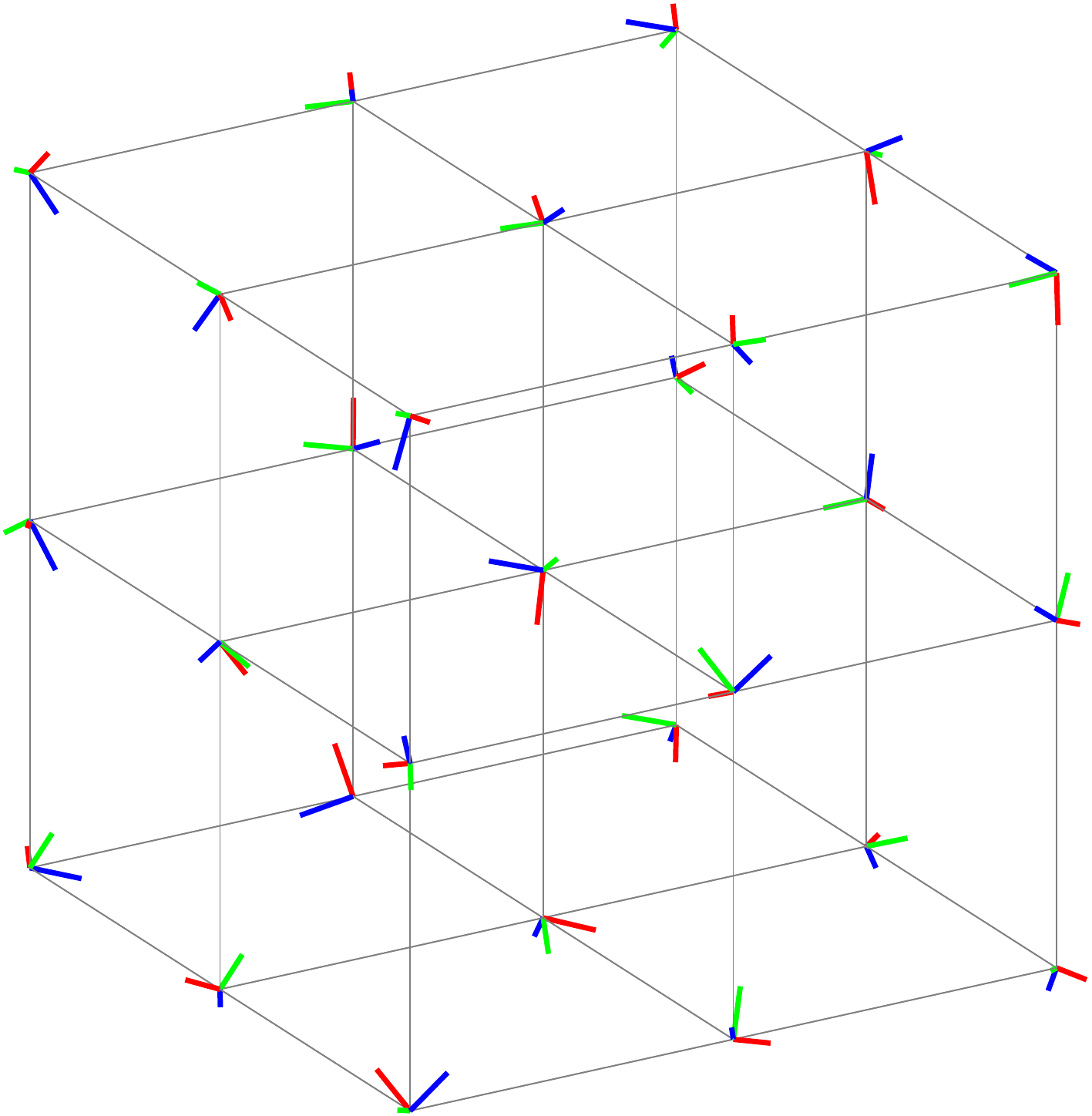}}
	\subfigure{\includegraphics[width=0.18\linewidth]{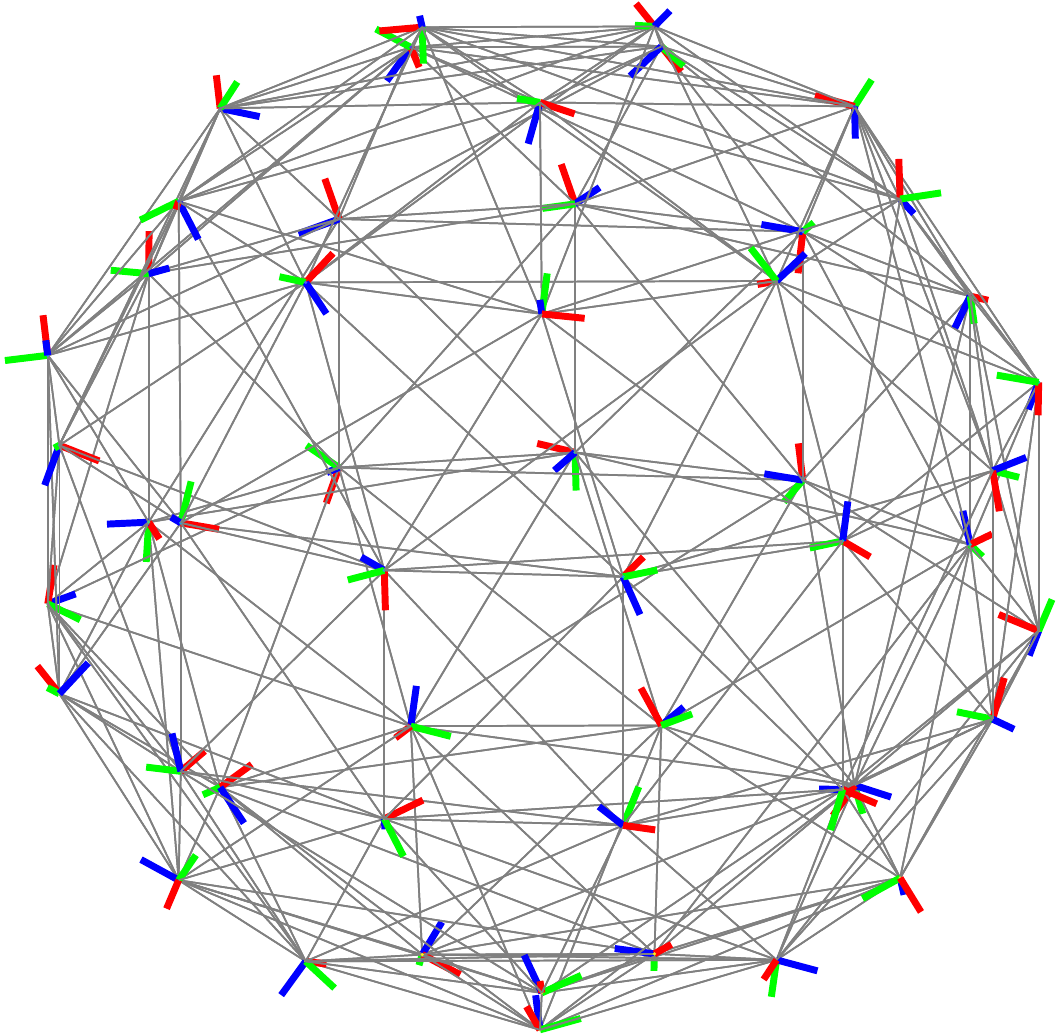}}
	\subfigure{\includegraphics[width=0.18\linewidth]{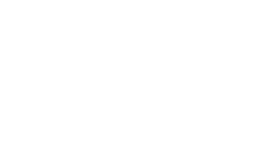}}\\
	\subfigure{\raisebox{15mm}{\rotatebox[origin=t]{90}{SE-Sync~\cite{rosen2016se}}}\hspace{5mm}}
	\subfigure{\includegraphics[width=0.18\linewidth]{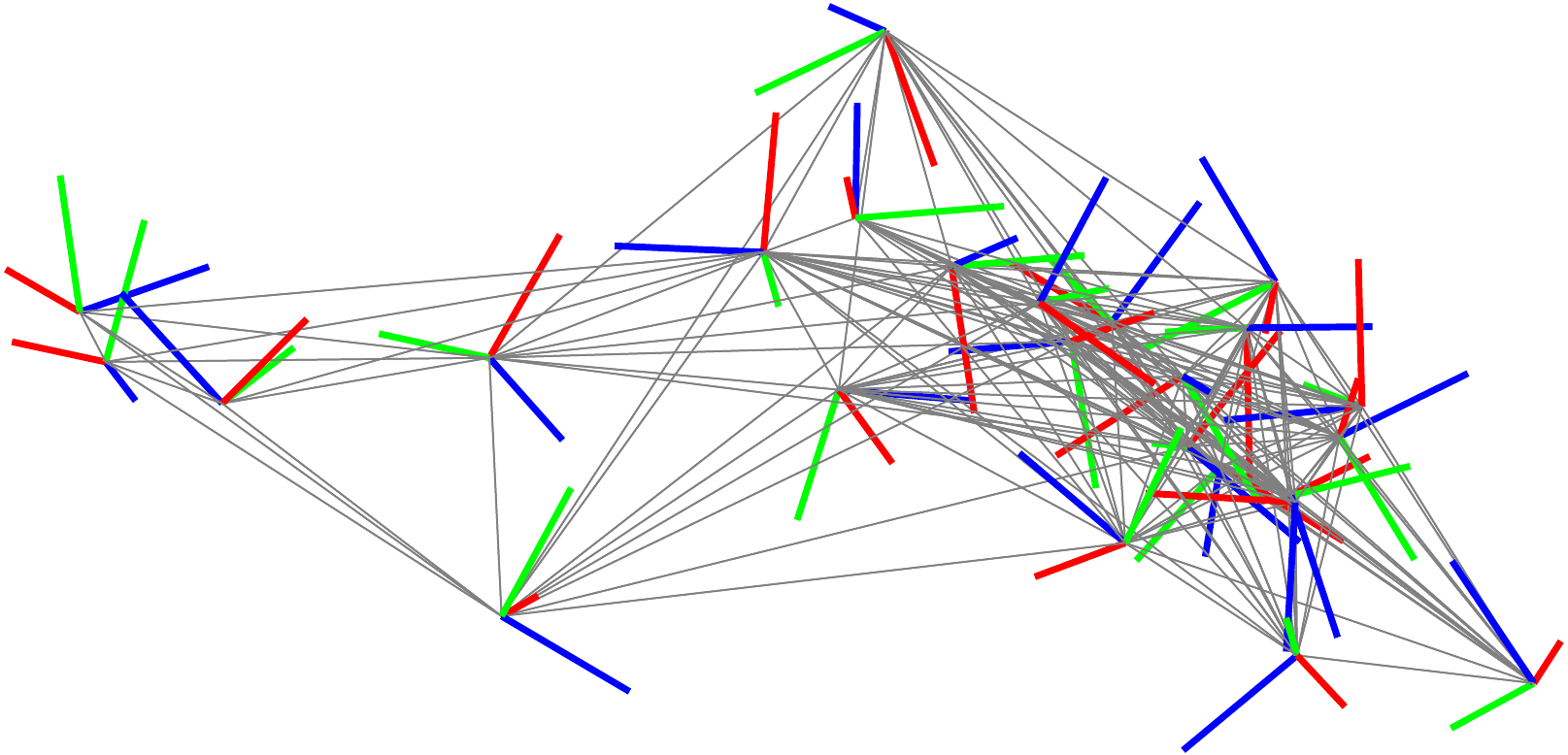}}
	\subfigure{\includegraphics[width=0.18\linewidth]{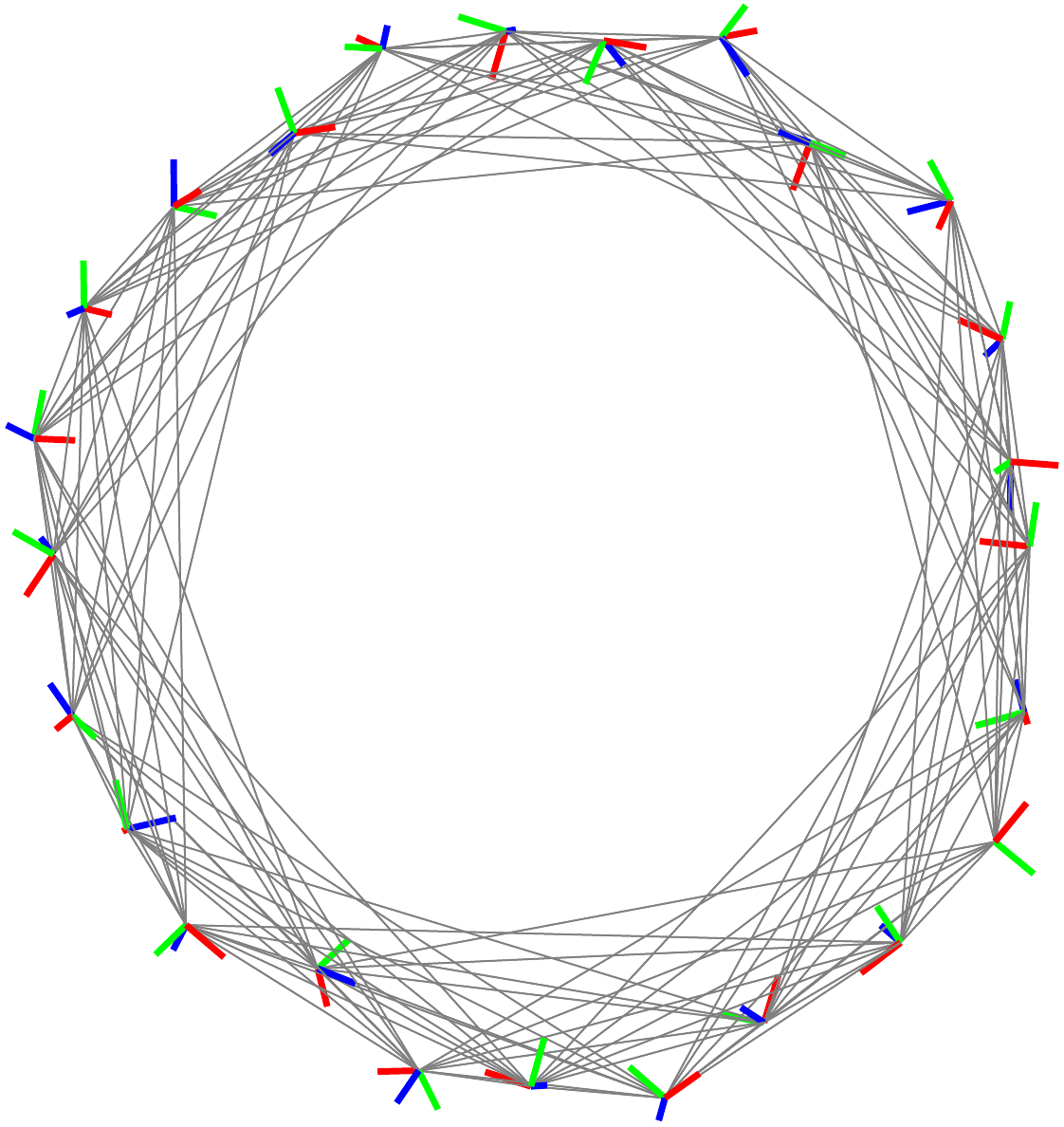}}
	\subfigure{\includegraphics[width=0.18\linewidth]{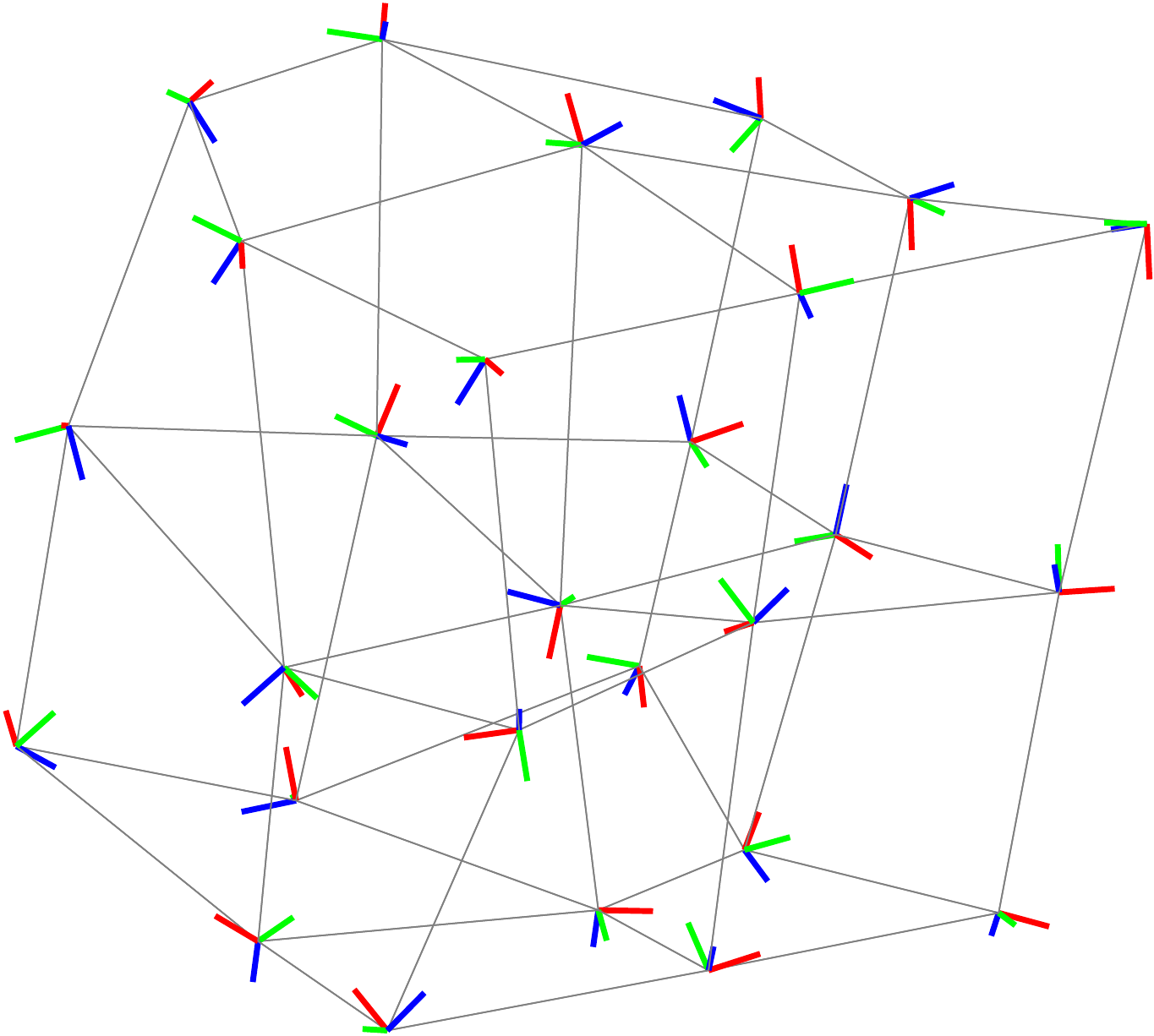}}
	\subfigure{\includegraphics[width=0.18\linewidth]{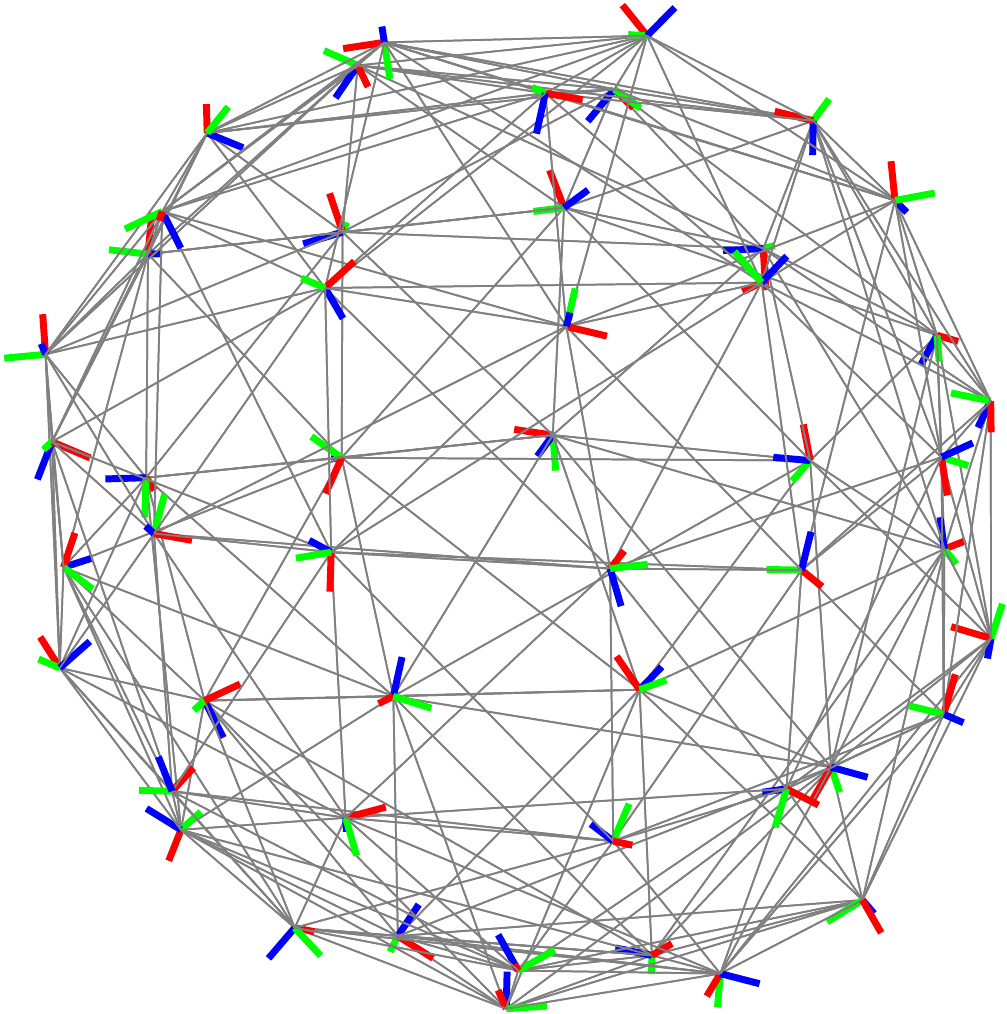}}
	\subfigure{\includegraphics[width=0.18\linewidth]{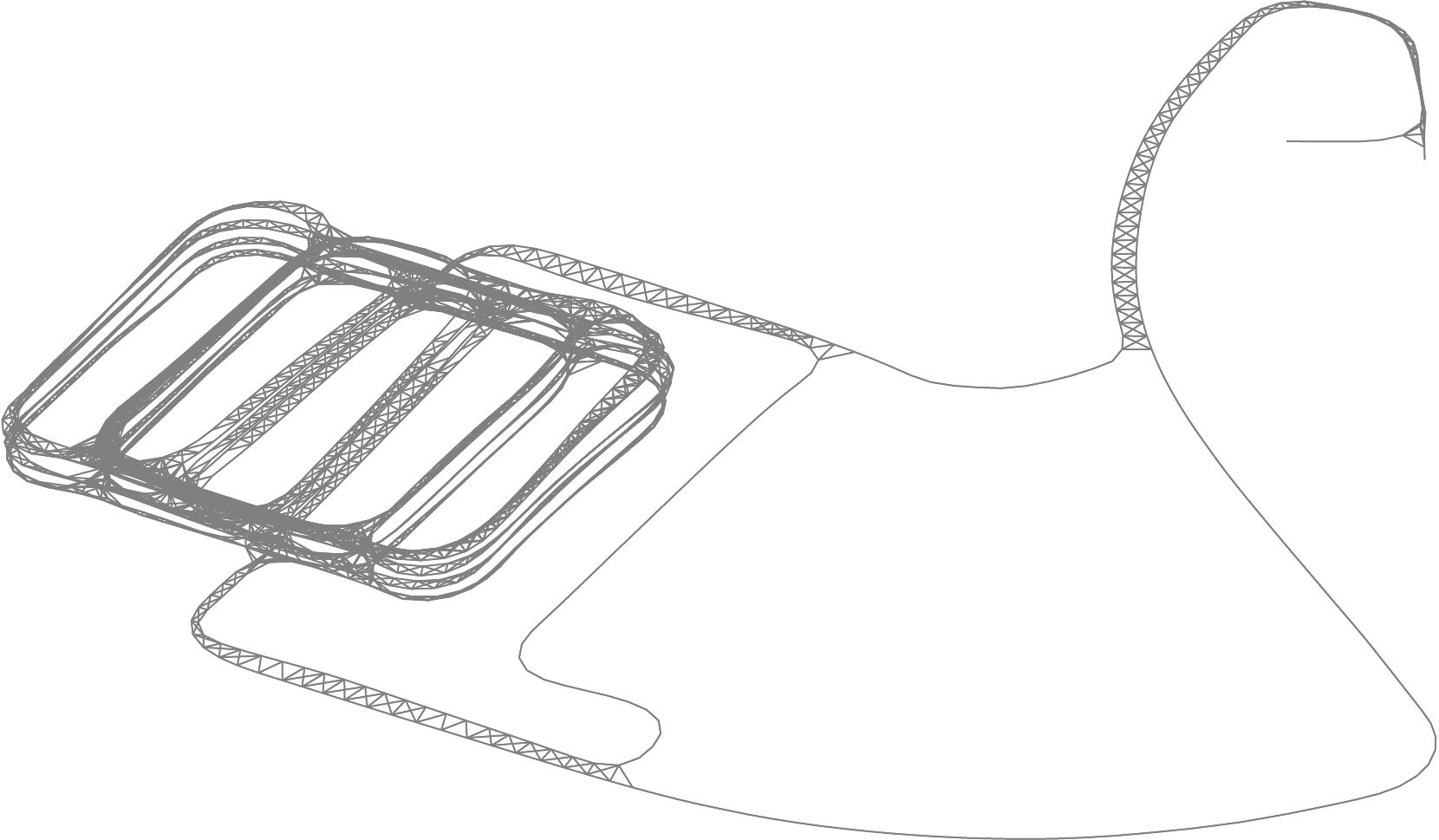}}\\
	\subfigure{\raisebox{15mm}{\rotatebox[origin=t]{90}{DGS~\cite{choudhary2017distributed}}}\hspace{5mm}}
	\subfigure{\includegraphics[width=0.18\linewidth]{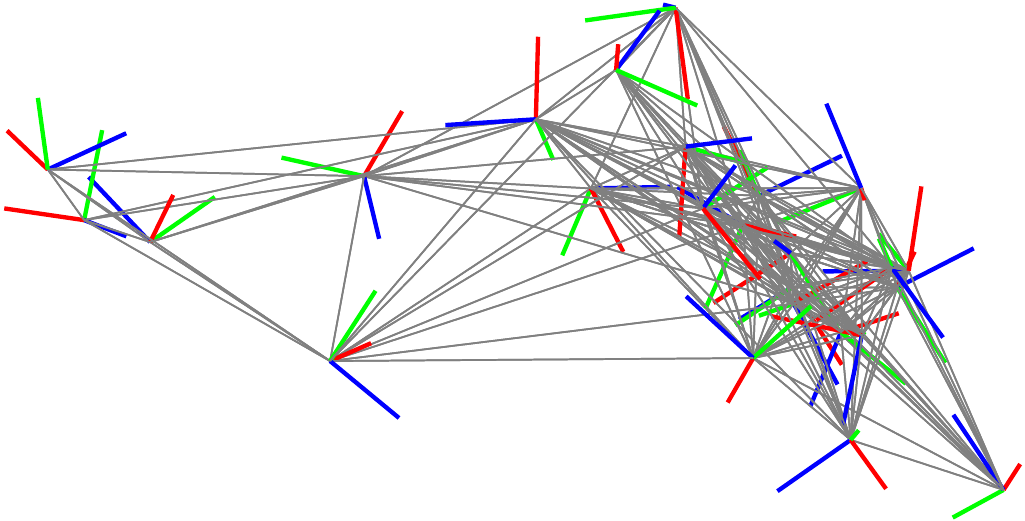}}
	\subfigure{\includegraphics[width=0.18\linewidth]{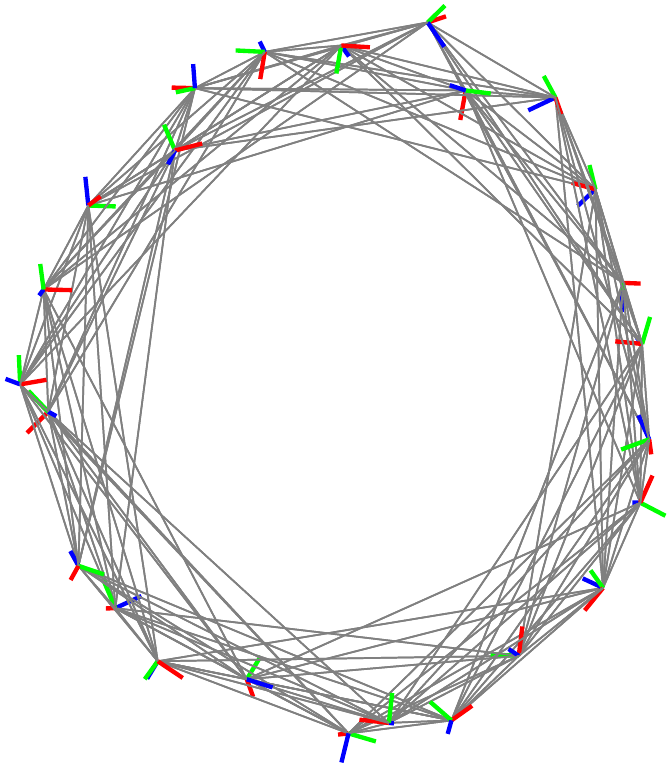}}
	\subfigure{\includegraphics[width=0.18\linewidth]{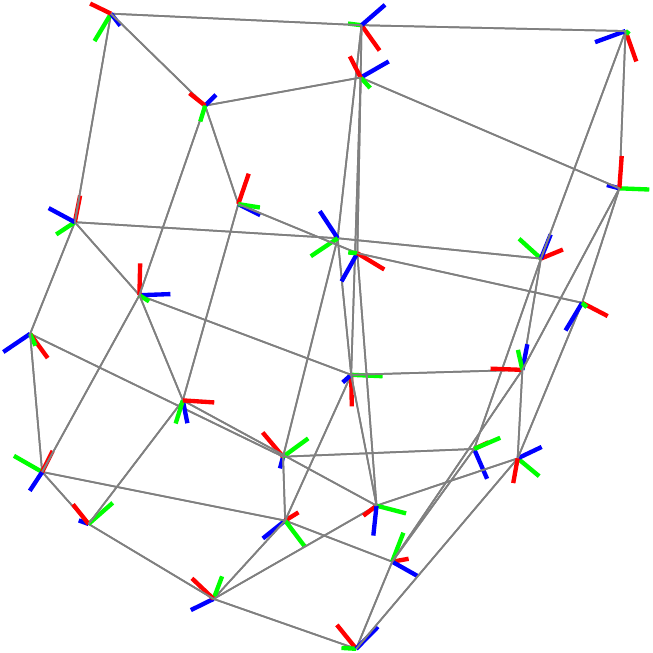}}
	\subfigure{\includegraphics[width=0.18\linewidth]{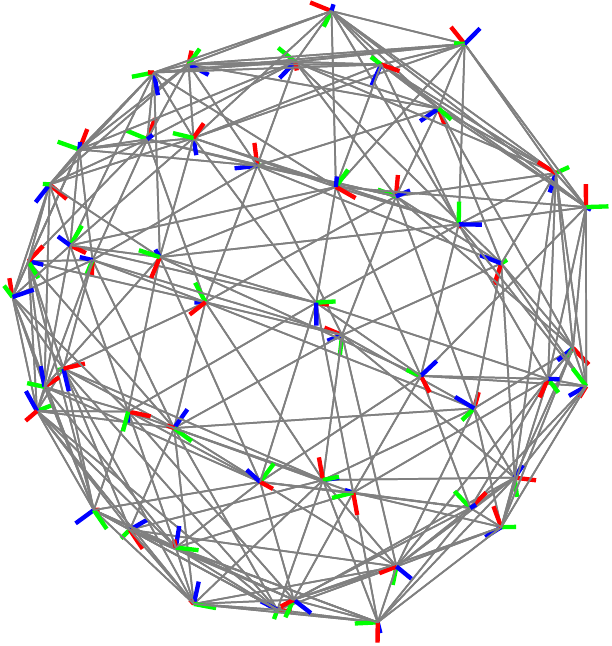}}
	\subfigure{\includegraphics[width=0.18\linewidth]{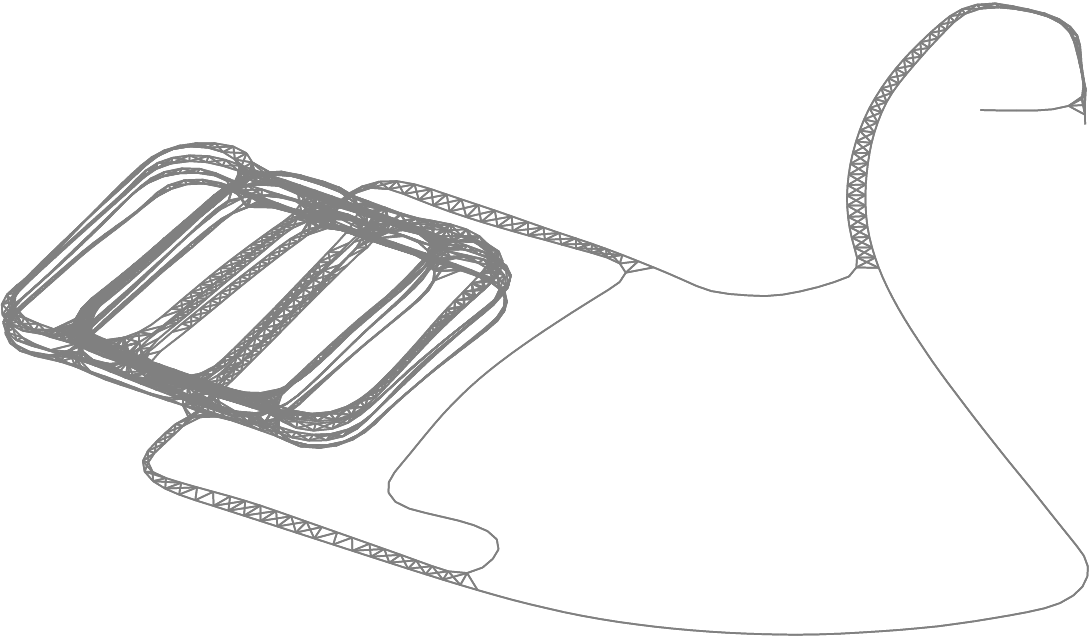}}\\
    \subfigure{\raisebox{15mm}{\rotatebox[origin=t]{90}{GeoD [This Paper]}}\hspace{5mm}}
	\subfigure{\includegraphics[width=0.18\linewidth]{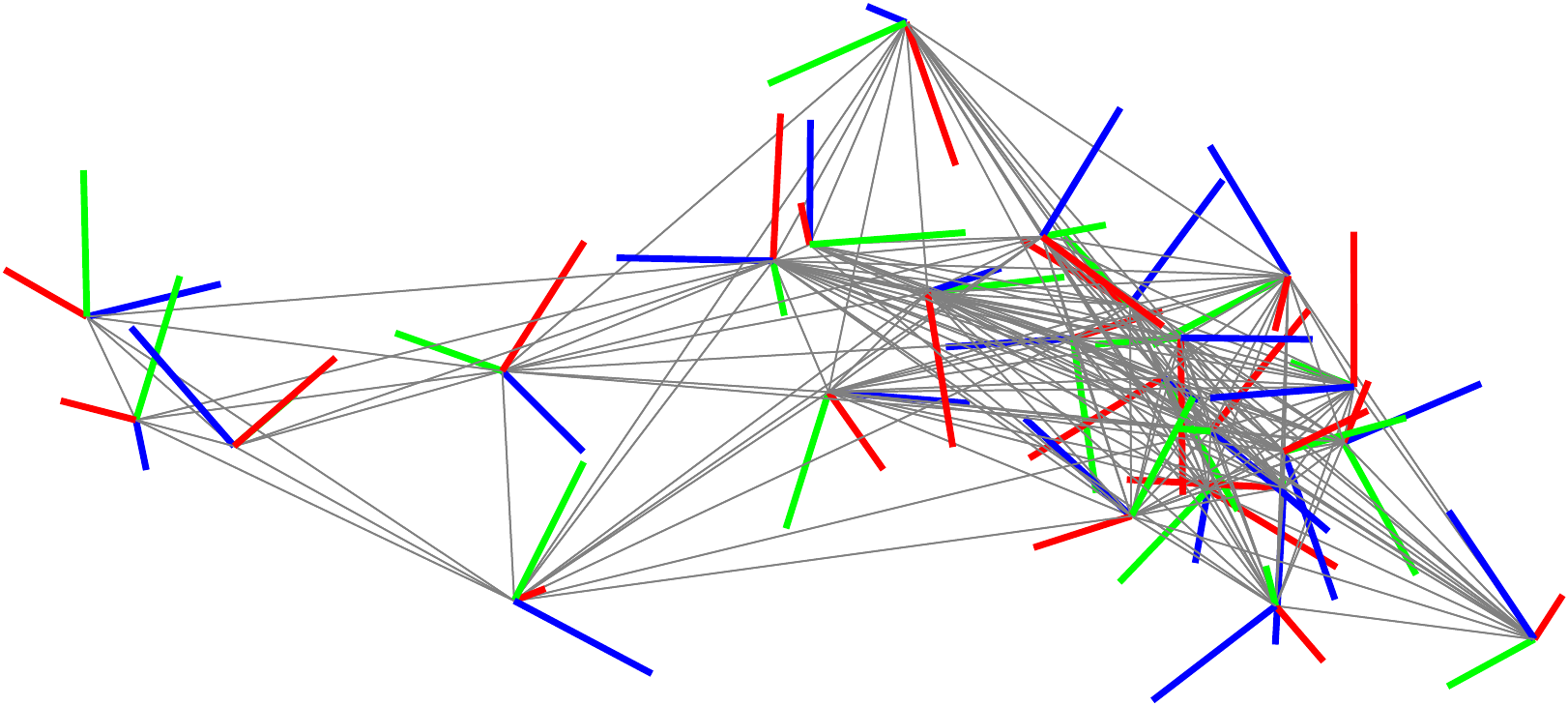}}
	\subfigure{\includegraphics[width=0.18\linewidth]{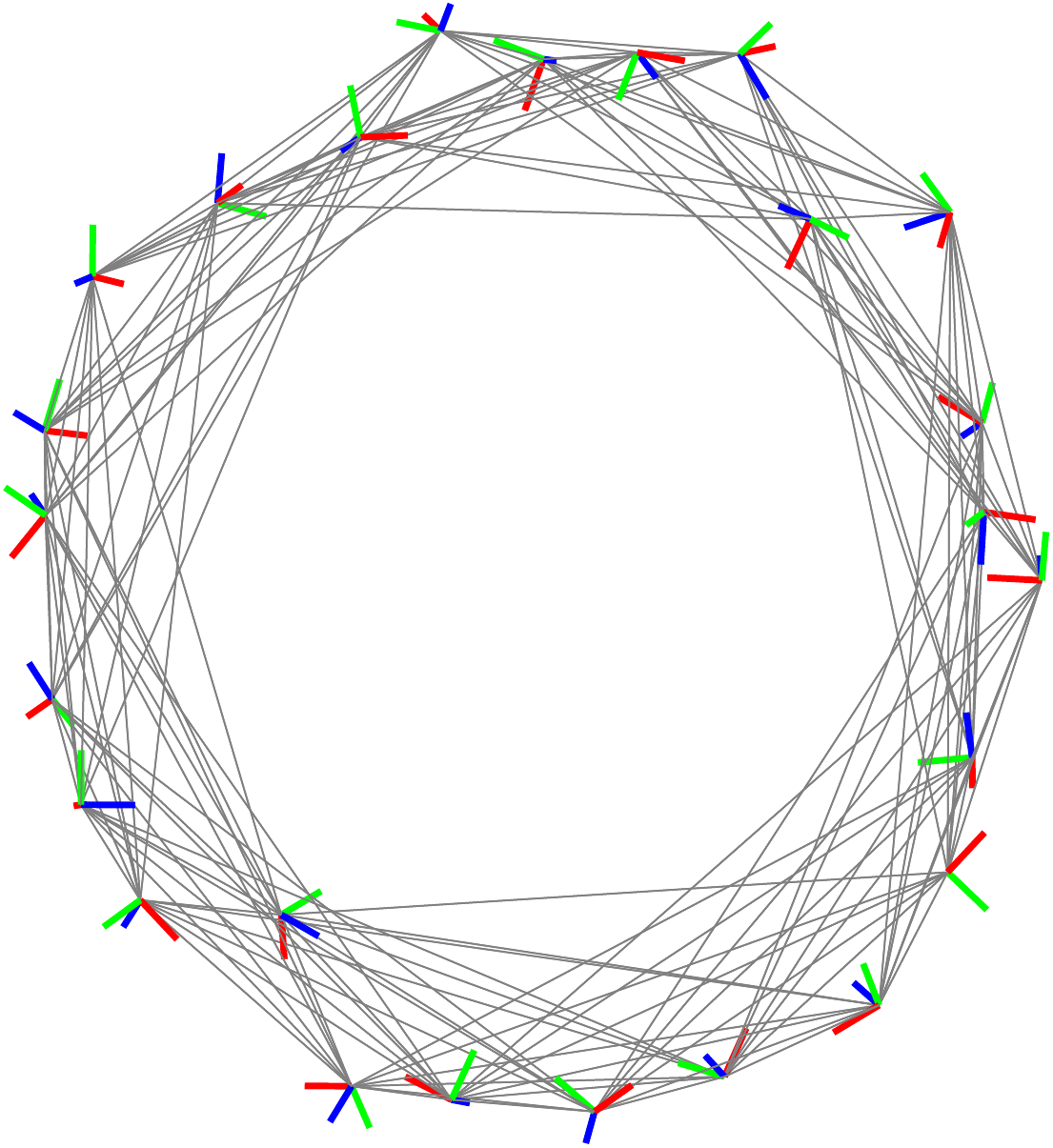}}
	\subfigure{\includegraphics[width=0.18\linewidth]{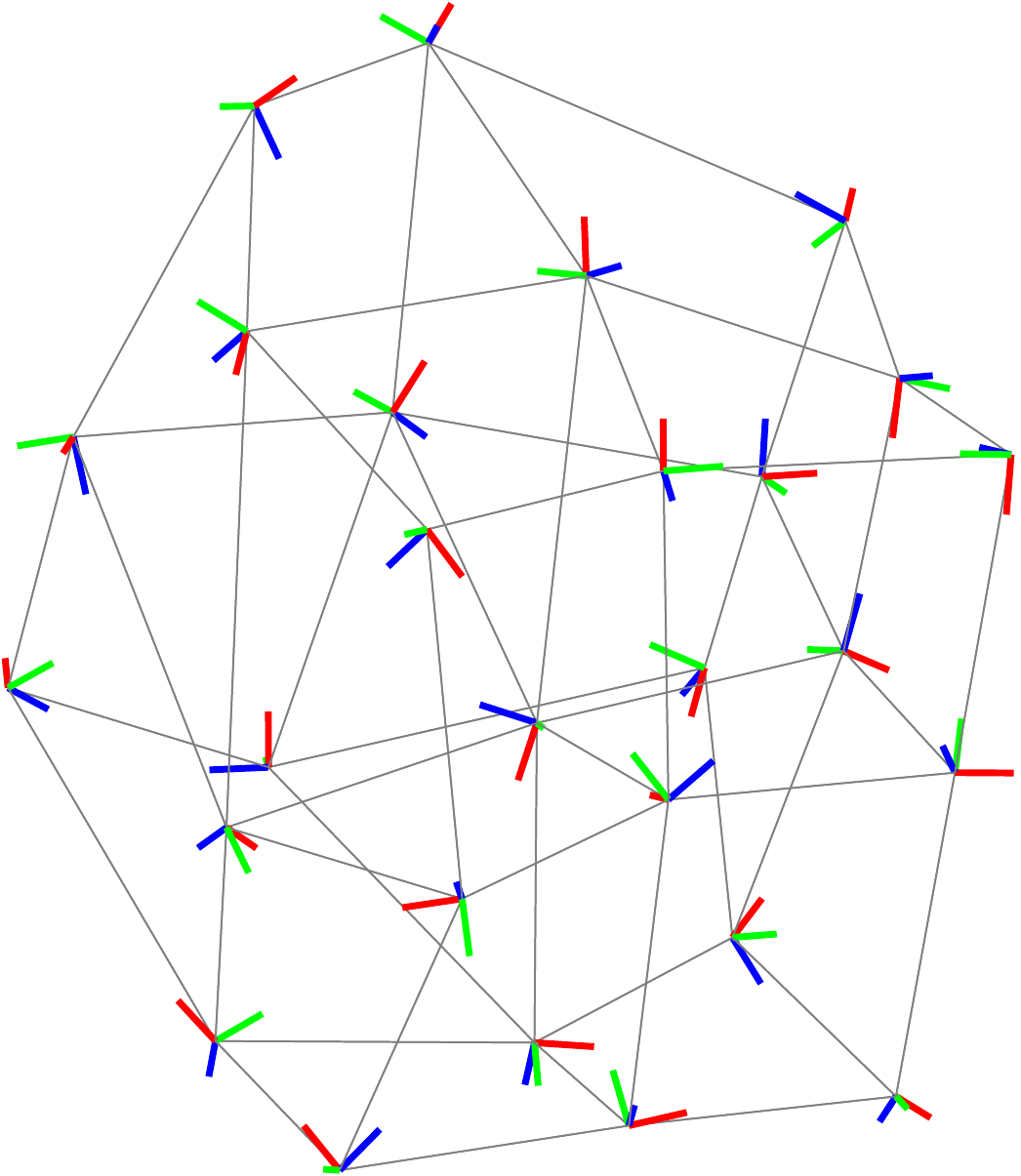}}
	\subfigure{\includegraphics[width=0.18\linewidth]{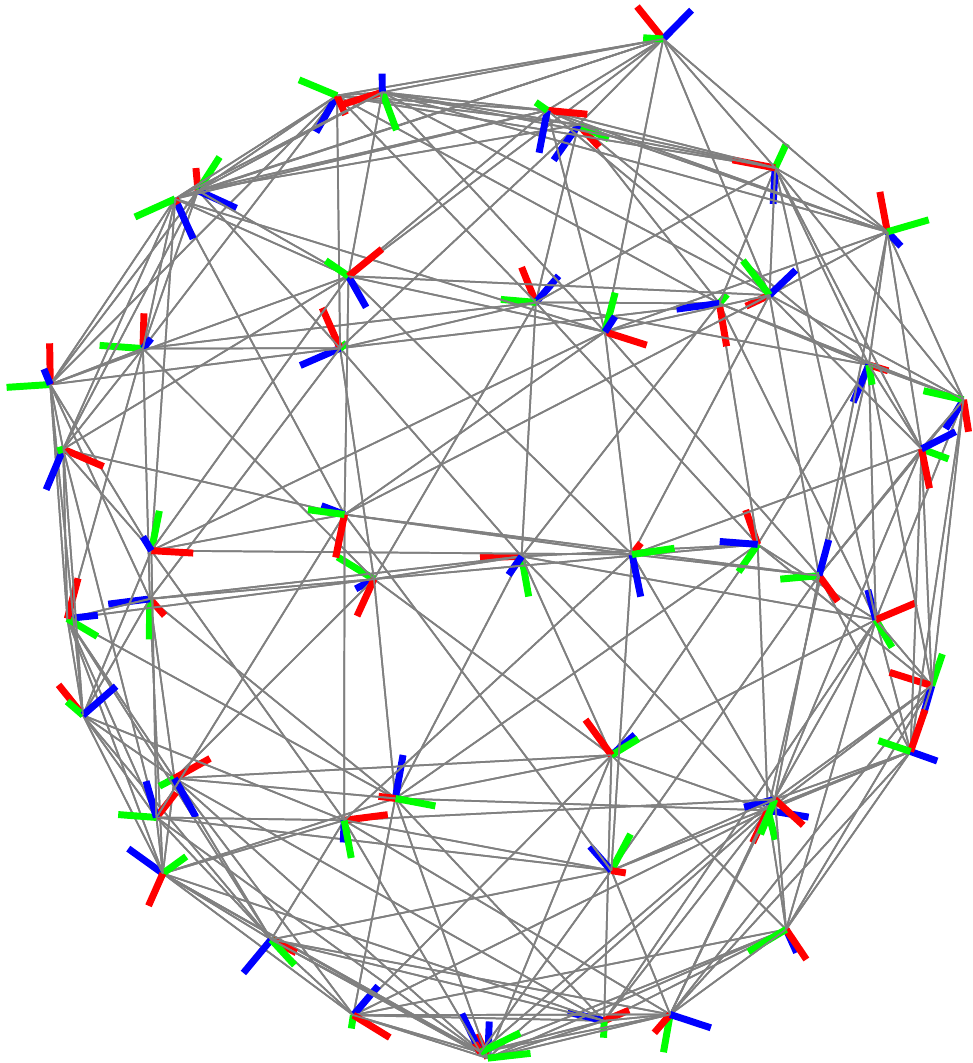}}
	\subfigure{\includegraphics[width=0.18\linewidth]{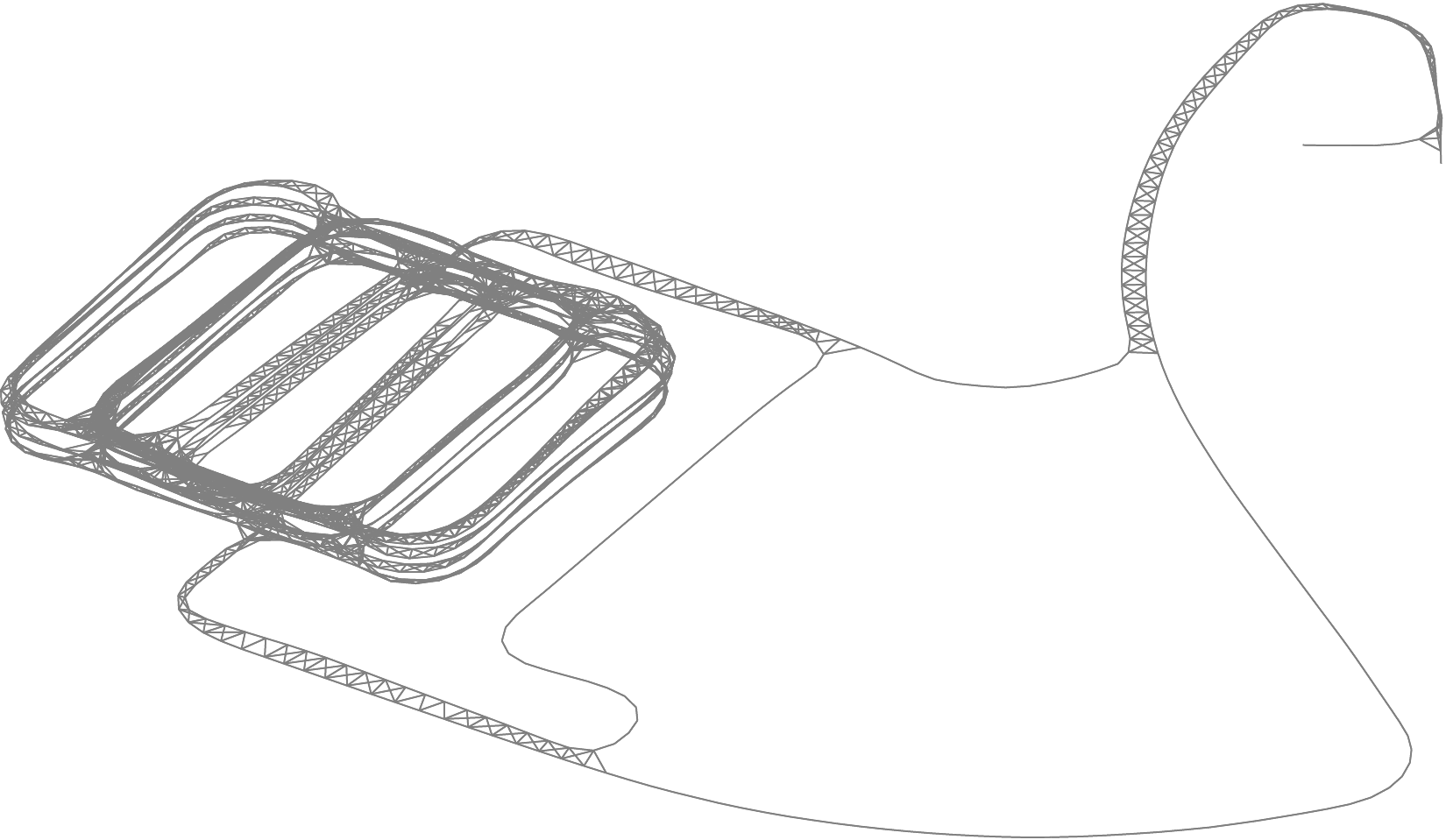}}\\
	\subfigure{\hspace{2.75mm}   \begin{picture}(0,5) \put(0,0){Random}   \end{picture}}
    \subfigure{\hspace{32mm}    \begin{picture}(0,5) \put(0,0){Circle}   \end{picture}}
    \subfigure{\hspace{32mm}    \begin{picture}(0,5) \put(0,0){Grid}     \end{picture}}
    \subfigure{\hspace{32mm}    \begin{picture}(0,5) \put(0,0){Sphere}   \end{picture}}
    \subfigure{\hspace{32mm}    \begin{picture}(0,5) \put(0,0){Garage}   \end{picture}}
\caption{Several pose graphs used in the numerical analysis are shown above; from left to right are the random, circular, grid, sphere, and parking-garage pose graphs. The ground truth network topologies are shown in the first row (the parking-garage does not include ground truth data). The global minimum solutions provided by SE-Sync are shown in the second row. The solutions obtained by DGS are shown in the third row. The solutions obtained by GeoD are shown in the fourth row. Our distributed solution reaches these results on average over $700$ times faster than the time to compute the centralized solution and is approximately $100\times$ faster than DGS on large networks like the parking-garage. Moreover, our method obtains a value that is on average 3.4$\times$ more accurate than DGS over all networks. The differences in value are hard to appreciate visually, but the grid is a good example where our approach yields a result closer to SE-Sync than DGS. 
}
\label{fig:results_networks}
\end{figure*}

\subsubsection{Consensus-based Distributed Comparison}
\label{sec:results_distributed}

We first compare our distributed estimation technique to the only other state-of-the-art consensus-based approach that formally considers the noisy relative measurements in the analysis which uses the angle-axis representation~\cite{cristofalo2019consensus}. 
We execute both methods on the sphere dataset with pairwise consistency in the measurements. We initialize both methods using the GPS technique as well. 
GeoD quickly converges to a geodesic value of $422.238$ in $32$ iterations ($0.0089$ $\mathrm{seconds}/\mathrm{agent}$) as calculated from the objective in Problem~\ref{prob:main}. The angle-axis approach achieves a geodesic value of $2237.248$ in $281$ iterations ($0.0743$ $\mathrm{seconds}/\mathrm{agent}$). 
The angle-axis method in \cite{cristofalo2019consensus} yields approximately $5\times$ greater pose estimation error and takes approximately $9\times$ longer than GeoD. We do not show plots of the resulting pose graphs due to space limitations, however the pose graph using \cite{cristofalo2019consensus} appears significantly farther form the ground truth poses than the pose graph from GeoD. 
This performance deficit is apparent in all other  datasets as well, hence we do not consider the method from \cite{cristofalo2019consensus} in the more extensive performance comparisons below. 

\subsubsection{State-of-the-art Comparison}
\label{sec:results_comparison}

The data in Table~\ref{tab:value_comparison} compares the chordal and geodesic values between all three methods, as well as the percent error of the chordal value with respect to the global minimum chordal value (denoted by a *). Table~\ref{tab:time_comparison} compares the run times of each method, as well as the communication iterations to convergence for the distributed methods. SE-Sync is a centralized algorithm and does not involve communication rounds, hence the tables show no iterations for SE-Sync. All methods are initialized using the GPS-like initialization, except for the parking-garage and cubicle datasets which are initialized using the spanning-tree initialization. SE-Sync reaches the certified global minimum chordal value in all cases. We report the following average statistics over all datasets except for the cubicle since DSG runs out of memory and could not return a result. Compared to the global minimum, our distributed method converges to a value with an average error of $32.5\%$, while the DGS approach converges to a value with an average error of $110.2\%$. It is important to recall that our method does not operate on the chordal error directly, so it should not be expected to reach the global minimum chordal value. Despite this fact, it still outperforms DGS. 

Our method is on average $717.27$ times faster than SE-Sync. 
On the small datasets, DGS is about $2-4\times$ faster than GeoD. 
However, DGS requires $100\times$ more computation on the garage dataset time compared to our method and exceeds our computer's memory capacity in the cubicle case, failing to produce a result. We believe this is due to the pre-processing required to set up and store the DGS problem for such a large number of agents. In contrast, our method runs~\eqref{eq:estimation_system} for each pose, which is more efficient, and scales well with graphs of thousands of poses. 
Although our method does not reach the global minimum, it reaches qualitatively representative local minima as shown in Fig.~\ref{fig:results_networks} more quickly than with SE-Sync, and significantly more quickly than DGS for large networks. 


\begin{table*}
\caption{Pose graph optimization objective value results. 
\textnormal{We compare our distributed consensus-based method, GeoD, to those from a state-of-the-art certifiable centralized method, SE-Sync, and a state-of-the-art distributed method, DGS. SE-Sync and DGS are convex relaxations that utilize the chordal distance instead of the geodesic distance, hence we report both values in this table. SE-Sync provides the certified global minimum chordal value which is indicated by the *. In all scenarios, our method reaches values that are closer to the global minimum compared to DGS, despite using different objective function. Note DGS runs out of memory on the large datasets with more than one thousand poses. } }
\label{tab:value_comparison}
\begin{center} \begin{tabular}{| ccc | cc  | ccc | ccc |}
    \hline
    & 	&	& \multicolumn{2}{c|}{SE-Sync~\cite{rosen2016se}}		& \multicolumn{3}{c|}{DGS~\cite{choudhary2017distributed}}	 &
    \multicolumn{3}{c|}{GeoD [This Paper]} \\
    Scene	& \# Poses	& \# Measurements	& Geodesic	& Chordal*	& Geodesic	& Chordal   & \% Error from * & Geodesic	& Chordal & \% Error from * \\
    \hline
    random	& 25	& 390	& 284.205 & 407.277 & 349.717 & 505.180 & 24.0\% & 286.315 & 409.615 & 0.6\% \\
    circle	& 25	& 300	& 221.367 & 315.298 & 304.577 & 415.762 & 31.9\% & 239.830 & 332.776 & 5.5\% \\
    grid	& 27	& 108	& 58.479 & 85.774 & 88.417 & 123.244 & 43.7\% & 66.179 & 92.082 & 7.4\% \\
    sphere	& 50	& 544 	& 387.896 & 558.386 & 496.650 & 688.637 & 23.3\% & 422.238 & 589.080 & 5.5\% \\
    experiment  & 217 & 508 & 78.056 & 99.635 & 296.637 & 435.768 & 337.4\% & 86.922 & 133.413 & 33.9\% \\
    garage	& 1,661	& 6,275	& 1.250 & 1.263 & 3.533 & 3.801 & 201.0\% & 3.055 & 3.056 & 142.1\% \\
    cubicle	& 5,750	& 16,869	& 644.799 & 717.126  & --- & --- & --- & 1234.126 & 1324.659 & 84.7\% \\
    \hline
\end{tabular} \end{center}
\end{table*}

\begin{table*}
\caption{Pose graph optimization computation time results. 
\textnormal{We compare GeoD's timing performance to SE-Sync and DGS. In all cases, our method outperforms SE-Sync. For small networks less than 50 agents, DGS is only faster by a factor of $2-4\times$. For larger networks, GeoD easily outperforms DGS, indicating that it scales well with increasing number of poses. In fact, our DGS implementation runs out of memory on large networks with greater than one thousand poses, illustrating GeoD's ability to successfully scale and handle pose graphs with thousands of poses. } }
\label{tab:time_comparison}
\begin{center} \begin{tabular}{| ccc | cc  | cc | cc |}
    \hline
    & 	&	& \multicolumn{2}{c|}{SE-Sync~\cite{rosen2016se}}		& \multicolumn{2}{c|}{DGS~\cite{choudhary2017distributed}}	 &
    \multicolumn{2}{c|}{GeoD [This Paper]} \\
    Scene	& \# Poses	& \# Measurements & Iterations	& Time [$\mathrm{s}$]	& Iterations	& Time/Agent [$\mathrm{s}$]	& Iterations	& Time/Agent [$\mathrm{s}$]\\
    \hline
    random	& 25	& 390   & --- & 0.1750 & 9 & 0.0054 & 17 & 0.0189 \\
    circle	& 25	& 300	& --- & 0.0975 & 15 & 0.0027 & 16 & 0.0044 \\
    grid	& 27	& 108	& --- & 0.0593 & 10 & 0.0013 & 55 & 0.0068 \\
    sphere	& 50	& 544	& --- & 0.1134 & 10 & 0.0020 & 32 & 0.0078 \\
    experiment & 217 & 508  & --- & 0.4836 & 141 & 0.0146 & 781 & 0.0609 \\
    garage	& 1,661	& 6,275	& --- & 84.7133 & 265 & 2.6211 & 94 & 0.0204 \\
    cubicle	& 5,750	& 16,869	& --- & 19.6280 & --- &--- & 5732 & 1.1467 \\
    \hline
\end{tabular} \end{center}
\end{table*}

\subsection{Multi-robot Pose Estimation Experiment}
\label{sec:results_experiment}

We compare approaches in an experiment using real quadrotors with forward-facing cameras. The purpose of this experiment is to show that our algorithm can estimate the poses of multiple robots' trajectories using only raw images collected during flight. Our algorithm then estimates the poses of the quadrotors as they moved through the scene, informing each robot where it is and where its neighbors are throughout the flight. 
We fly seven autonomous quadrotors around an object of interest while simultaneously collecting images from their cameras (see Fig.~\ref{fig:experiment}). After collection, the quadrotors extact image features and match them with their neighboring robots using classic feature-based Structure from Motion (SfM) methods. We additionally use SfM to estimate the noisy relative poses between inter and intra quadrotor trajectories. This information is the input to GeoD as the graph of noisy pose measurements. Here, we assume each quadrotor is responsible for computing the GeoD iterations (\ref{eq:estimation_system}) for all poses within its own trajectory. 

We initialize the pose optimization by adding artificial noise to the poses obtained from an external motion capture system during their flight, to simulate the GPS-like initialization used throughout this work. In this experiment the parameters for the initial condition noise are $\tau = 0.1 \ \mathrm{meters}$ and $\kappa = 0.175\ \mathrm{radians}$ or $10\ \mathrm{degrees}$. GeoD reaches a final geodesic value of $86.922$ (compared to the geodesic value associated to the chordal global minimum of $78.056$) in $0.0609$ $\mathrm{seconds}/\mathrm{agent}$ (compared to $0.4836\ \mathrm{seconds}$ for SE-Sync). In contrast, the distributed DGS method reaches a final geodesic value of $296.637$ in $0.0146$ $\mathrm{seconds}/\mathrm{agent}$. The final pose graphs are visualized around the SfM $3$D reconstruction in Fig.~\ref{fig:experiment}. The pose graphs are aligned to the reconstruction afterwards to illustrate the resulting scale of the solution. 
Interestingly, the converged solution of GeoD visually represents the ground truth trajectories better than SE-Sync. We suspect this indicates that the geodesic value captures important geometrical qualities. 

\begin{figure*}
\centering
    \subfigure[Ground truth (top)\label{fig:experiment_ground_truth_top}]{\raisebox{5mm}{\includegraphics[width=0.24\linewidth]{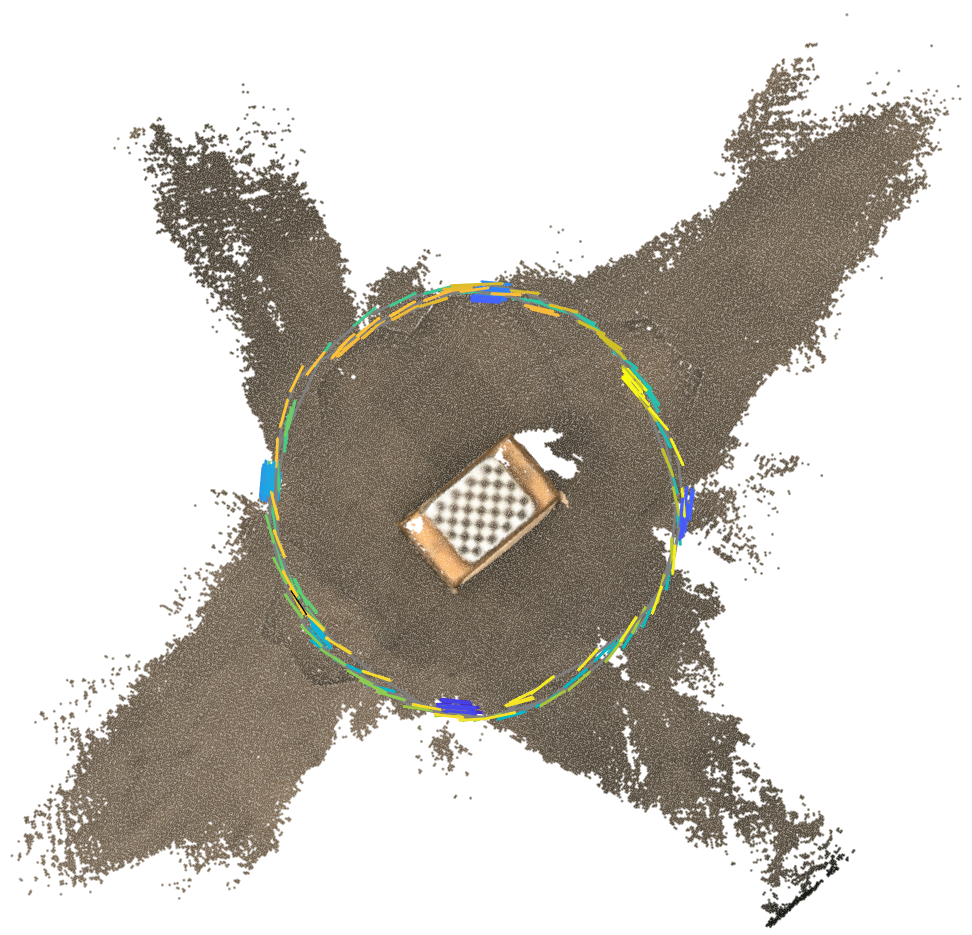}}}
	\subfigure[SE-Sync~\cite{rosen2016se} (top)\label{fig:experiment_sesync_top}]{\raisebox{0mm}{\includegraphics[width=0.24\linewidth]{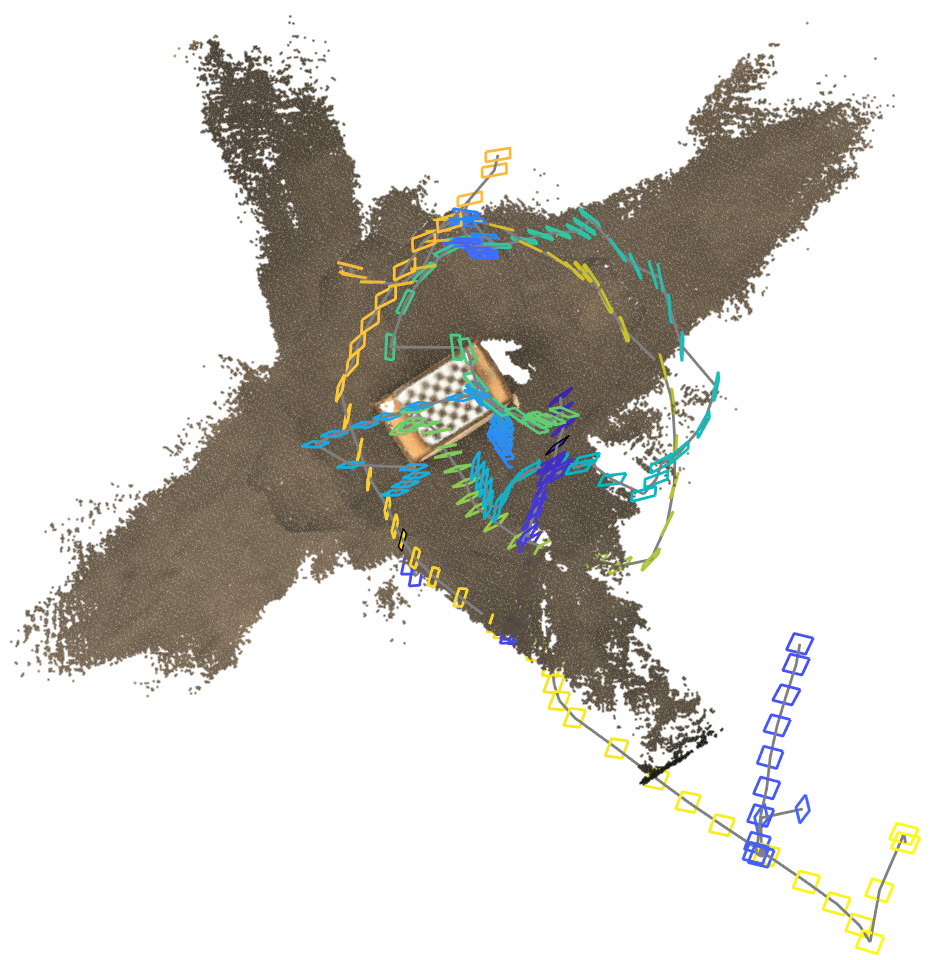}}} \subfigure[DGS~\cite{choudhary2017distributed} (top)\label{fig:experiment_jor_top}]{\raisebox{5mm}{\includegraphics[width=0.24\linewidth]{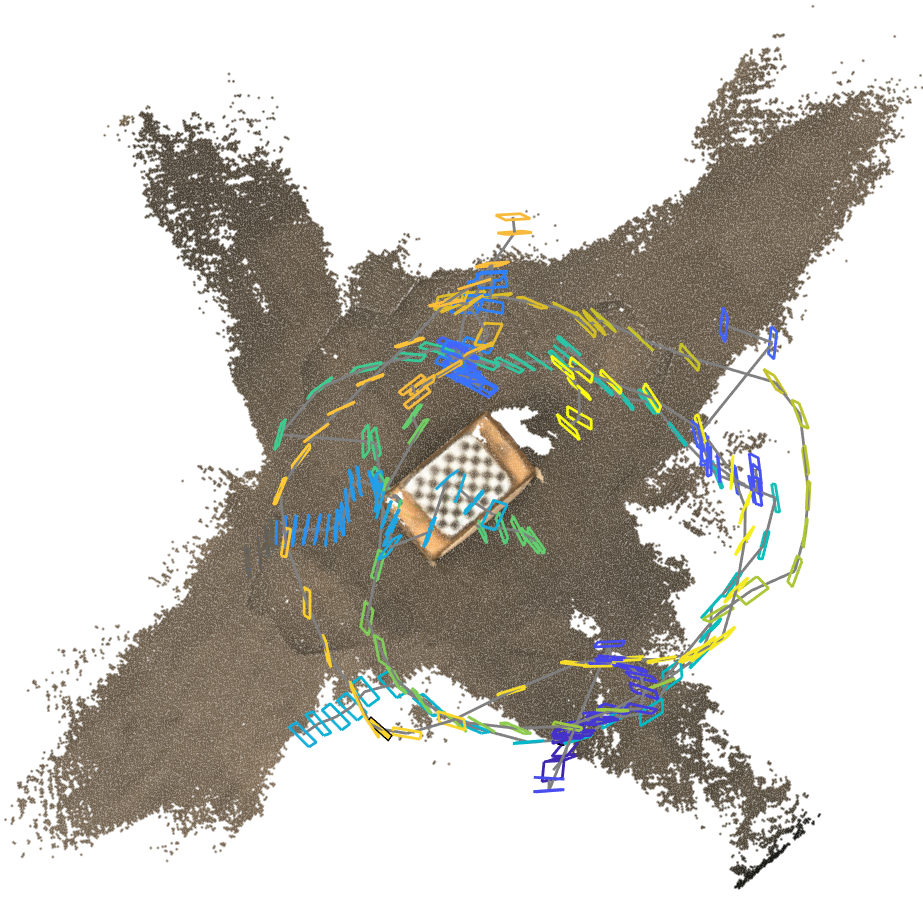}}}
    \subfigure[GeoD  (top)\label{fig:experiment_geodesic_top}]{\raisebox{5mm}{\includegraphics[width=0.24\linewidth]{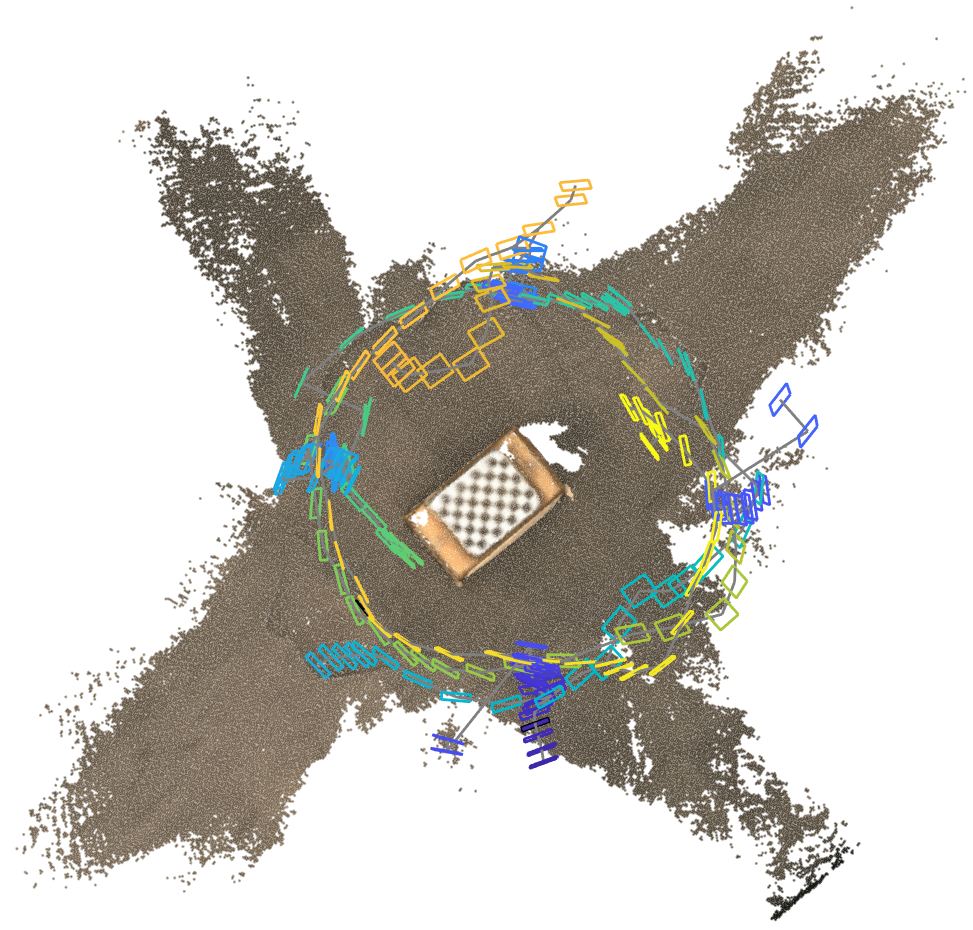}}}
    \\
    \subfigure[Ground truth (side)\label{fig:experiment_ground_truth}]{\raisebox{4.9mm}{\includegraphics[width=0.24\linewidth]{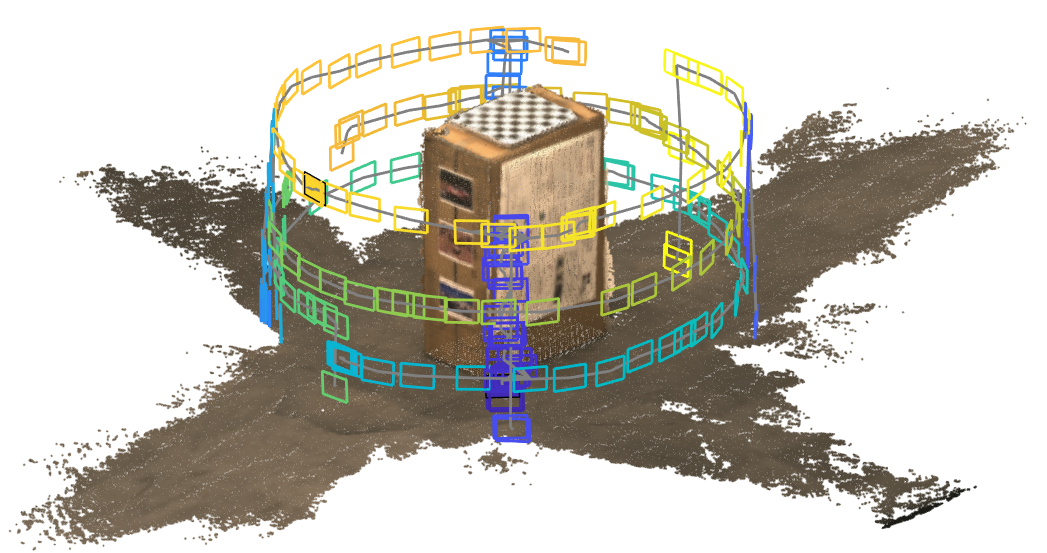}}}
	\subfigure[SE-Sync~\cite{rosen2016se} (side)\label{fig:experiment_sesync}]{\raisebox{0mm}{\includegraphics[width=0.24\linewidth]{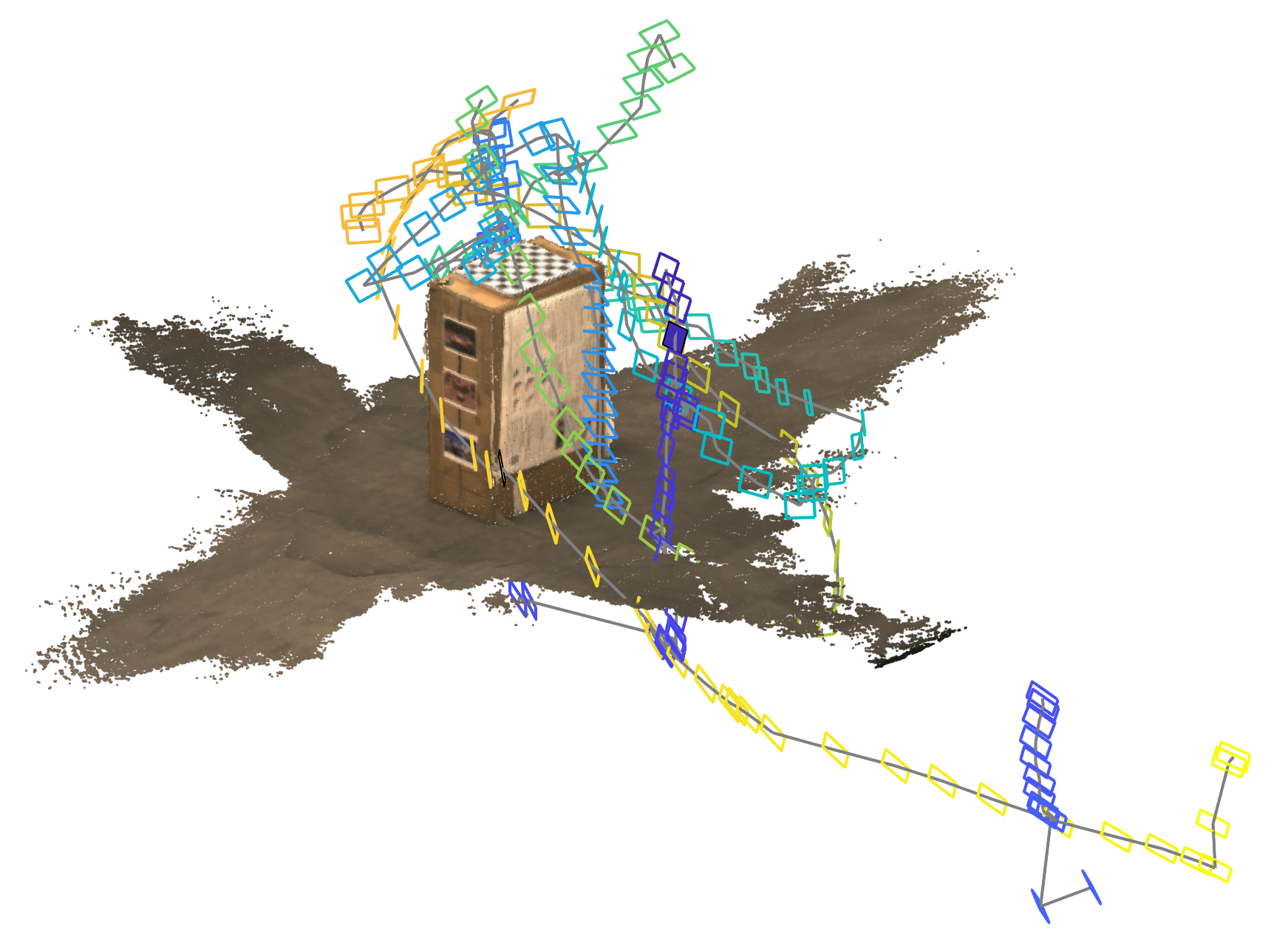}}} \subfigure[DGS~\cite{choudhary2017distributed} (side)\label{fig:experiment_jor}]{\raisebox{0mm}{\includegraphics[width=0.24\linewidth]{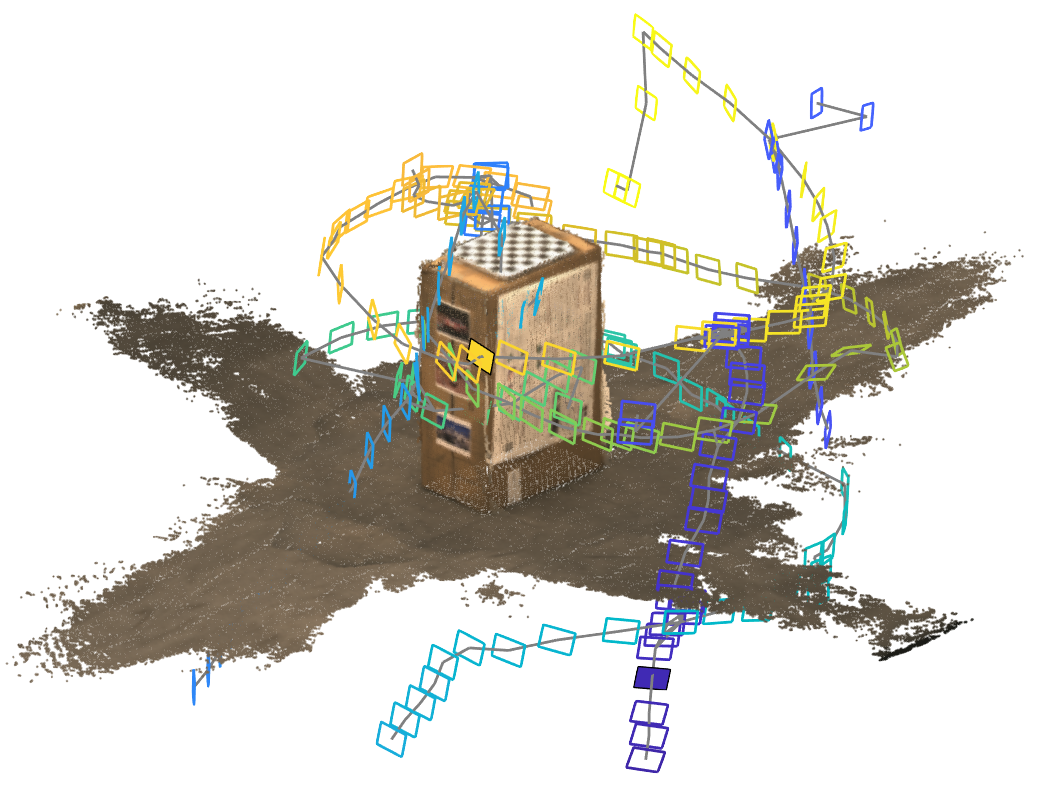}}}
    \subfigure[GeoD  (side)\label{fig:experiment_geodesic}]{\raisebox{4.9mm}{\includegraphics[width=0.24\linewidth]{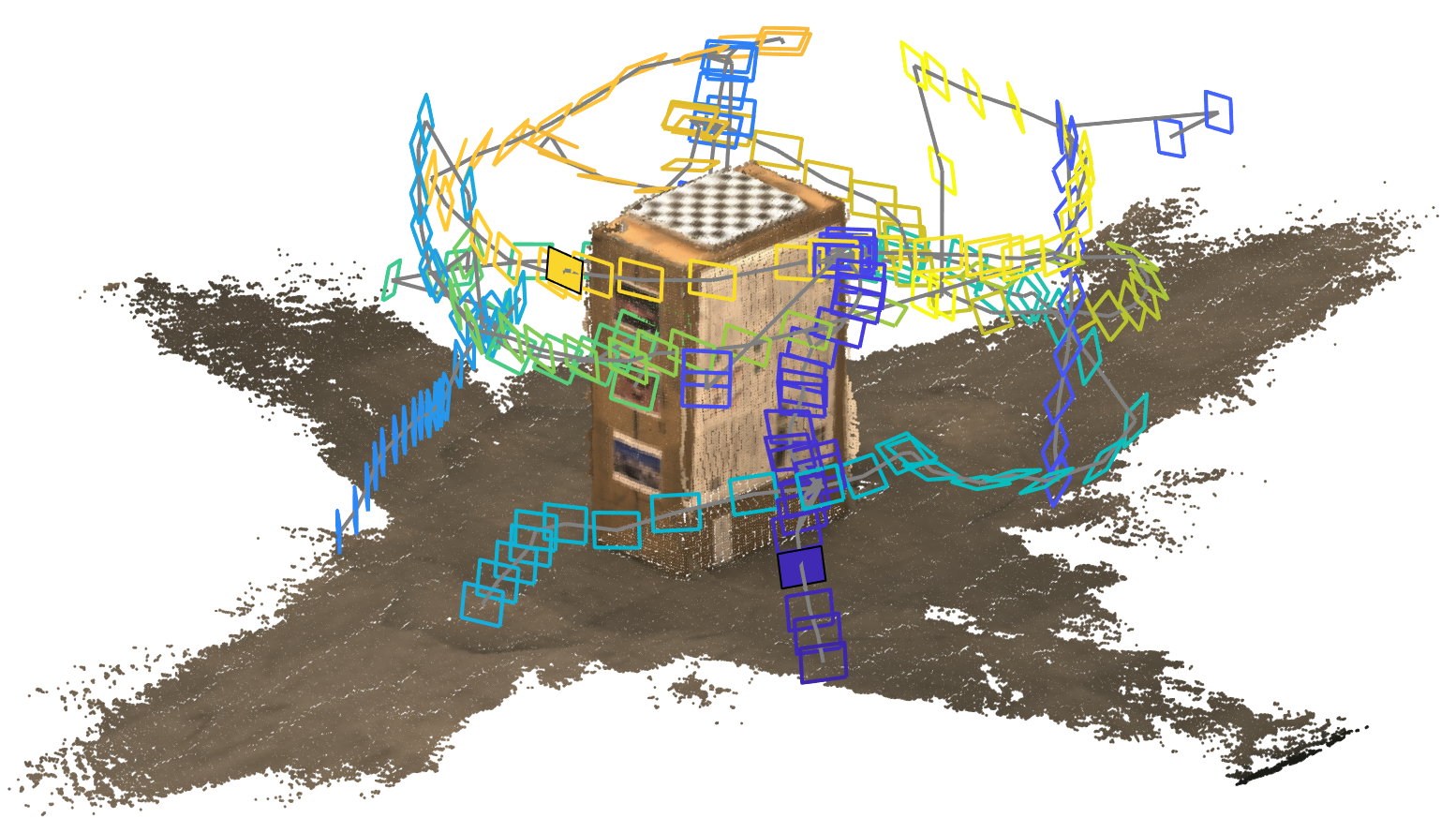}}}
\caption{We deploy GeoD on a team of vision-equipped quadrotors to compare their images with their neighbors, extract noisy relative pose measurements, and reach agreement on the quadrotors' poses in a distributed manner. Above are examples of the camera poses in the 3D reconstructed scene for the ground truth (a, e), SE-Sync (b, f), DGS (c, g), and GeoD (d, h). The figures in the top row and bottom row are viewed from an aerial view and side view, respectively. We aligned the resulting solutions to the $3$D reconstruction by anchoring the first camera pose to the ground truth coordinate frame. Our method outperforms the DGS by converging to a solution with less value ($86.922$ compared to $296.637$) and also qualitatively better represents the ground truth acquired via motion capture. Notice that the trajectories provided GeoD visually correspond to the ground truth, while the other methods return pose graphs that intersect the ground plane. 
}
\label{fig:experiment}
\end{figure*}

\section{Conclusion}\label{sec:conclusion}
We have presented GeoD, a consensus-based distributed pose graph optimization algorithm with provable convergence guarantees. 
GeoD utilizes the geodesic distance on rotation matrices rather than the chordal distance and is guaranteed to converge when the noisy measurements along the graph's edges are constistent. 
We have leveraged tools from Lyapunov theory and multi-agent consensus to prove that pairwise consistent relative measurements are a sufficient condition for convergence of the entire $SE(3)$ estimator. 

We have compared GeoD to state-of-the-art centralized and decentralized methods. Our gradient-based method reaches convergence in less time than the centralized method and reaches a more accurate solution that the distributed method. GeoD was additionally shown to scale better on pose graphs with large numbers of poses. We finally deployed our pose graph optimization method on real image data acquired from aerial drone trajectories. Our method allows the multi-UAV system to align their images and use this information to perform a vision-based $3$D reconstruction. 

We believe consensus-based estimators like GeoD are useful for their provable guarantees and their ease of implementation. In the future, we hope to find the conditions that ensure convergence to the global geodesic minimum and to deploy consensus-based methods as back-ends in an online distributed SLAM system. 

\appendix
\label{sec:appendix}

\subsection{Proof of Theorem~\ref{thm:time_derivative_dot_matrix_logarithm}}
\label{app:lie_group}

\begin{proof}
From the definition of $\theta$ in~\eqref{eq:angle_axis_theta}, we write 
the time derivative of the squared geodesic distance as a function of $\theta$ and $\dot\R$, 
\begin{equation}
	\frac{d}{dt} \frac{1}{2} \|\log(\R)^{\vee}\|_2^2 = \frac{d}{dt} \frac{1}{2} \theta^2 
	= \theta \dot\theta = \frac{-\theta\tr(\dot\R)}{2 \sin(\theta)}
	\, .
\end{equation}
We premultiply the $\dot\R$ inside the trace by the identity matrix $\R \R^T$ such that $\R (\R^T\dot\R)$ is a product of a rotation matrix, $\R$, and a skew-symmetric matrix, $(\R^T\dot\R)$, such that, 
\begin{equation}
	\frac{d}{dt} \frac{1}{2} \|\log(\R)^{\vee}\|_2^2 = \frac{-\theta\, \tr(\R (\R^T\dot\R))}{2 \sin(\theta)}
	\, .
\end{equation}
Note that $(\R^T\dot\R)\in so(3)$ due to the expression for the rate of change of a rotation matrix. 
Using Lemma~\ref{lem:skew_symmetric}, we find the resulting expression is a function of $\log(\R)$ and thus arrive at the desired expression, 
\begin{equation} \begin{aligned}
	\frac{d}{dt} \frac{1}{2} \|\log(\R)^{\vee}\|_2^2 
	&= \frac{\theta {(\R^T\dot\R)^{\vee}}^T(\R - \R^T)^{\vee} }{2 \sin(\theta)} \\
	&=  {(\R^T\dot\R)^{\vee}}^T \left( \frac{\theta}{2 \sin(\theta)} (\R - \R^T)\right)^{\vee} \\
	&= {(\R^T\dot\R)^{\vee}}^T \log(\R)^{\vee}
	\, .
\end{aligned} \end{equation}
\end{proof}

\subsection{Proofs of Results from Section~\ref{sec:analysis_rotation_matrix}}
\label{app:analysis_rotation_matrix}

\begin{proof}[Proof of Lemma~\ref{lem:rotated_control}]
	From the properties of the matrix Logarithm~\eqref{eq:log_products_result}, we know that the following is true for any $\mathbf{Q} \in \{ \R_{ij}, \tilde\R_{ij} \}$ and $\bm{\omega}_{ij}$, 
\begin{equation} \begin{aligned}
	\mathbf{Q}^T\bm\omega_{ij}^{\wedge}\mathbf{Q}
		&= \mathbf{Q}^T \log(\R_{ij}\tilde\R_{ij}^T) \mathbf{Q} \\
		&= \log(\mathbf{Q}^T \R_{ij} \tilde\R_{ij}^T \mathbf{Q} ) 
		\, .
\end{aligned} \end{equation}
The previous equation now takes the same form for each $\mathbf{Q} \in \{ \R_{ij}, \tilde\R_{ij} \}$, 
\begin{equation} \begin{aligned}
	\log(\tilde\R_{ij}^T \R_{ij} \tilde\R_{ij}^T \tilde\R_{ij} ) 
		&= \log(\R_{ij}^T \R_{ij} \tilde\R_{ij}^T \R_{ij} ) \\
		&= \log(\tilde\R_{ij}^T \R_{ij} )
		\, .
\end{aligned} \end{equation}
Finally, using the property of the matrix logarithm of a matrix inverse, we reach the desired result.
\begin{equation} \begin{aligned}
	\log(\tilde\R_{ij}^T \R_{ij} )
	&= - \log(\R_{ij}^T \tilde\R_{ij} ) \\
	&= - \log(\R_{ji} \tilde\R_{ji}^T ) \\
	&= - \bm\omega_{ji}^{\wedge} 
	\, ,
\end{aligned} \end{equation}
and $\mathbf{Q}^T\bm\omega_{ij}^{\wedge}\mathbf{Q} = -\bm\omega_{ji}^{\wedge}$ as intended. 
\end{proof}

\begin{proof}[Proof of Lemma~\ref{lem:dot_product_equivalence_1}]
	We reach the desired result using Lemma~\ref{lem:rotated_control},~\eqref{eq:dot_product_trace}, and the properties of the trace operator,  
	\begin{equation} \begin{aligned}
		\bm\omega_{ij}^T \left( \tilde\R_{ij} \bm\omega_{ik}^{\wedge} \tilde\R_{ij}^T \right)^{\vee} &= \frac{1}{2} \tr \left( {\bm\omega_{ij}^{\wedge}}^T \tilde\R_{ij} \bm\omega_{ik}^{\wedge} \tilde\R_{ij}^T \right) \\
		&= \frac{1}{2} \tr \left( \tilde\R_{ij}^T {\bm\omega_{ij}^{\wedge}}^T \tilde\R_{ij} \bm\omega_{ik}^{\wedge} \right) \\
		&= \frac{1}{2} \tr \left( -{\bm\omega_{ji}^{\wedge}}^T \bm\omega_{ik}^{\wedge} \right) \\
		&= -\bm\omega_{ji}^T\bm\omega_{ik}
		\, .
	\end{aligned} \end{equation}
\end{proof}

\begin{proof}[Proof of Lemma~\ref{lem:dot_product_equivalence_2}]
	Similarly to Lemma~\ref{lem:dot_product_equivalence_1}, we reach the desired result using the previous properties of the matrix Logarithm, 
	\begin{equation} \begin{aligned}
		&\bm\omega_{ij}^T \left( \tilde\R_{ij}\R_{ij}^T \bm\omega_{ik}^{\wedge} \R_{ij}\tilde\R_{ij}^T \right)^{\vee} \\
		&\qquad\qquad = \frac{1}{2} \tr \left( {\bm\omega_{ij}^{\wedge}}^T \tilde\R_{ij}\R_{ij}^T \bm\omega_{ik}^{\wedge} \R_{ij}\tilde\R_{ij}^T \right) \\
		&\qquad\qquad = \frac{1}{2} \tr \left( \R_{ij} \tilde\R_{ij}^T {\bm\omega_{ij}^{\wedge}}^T \tilde\R_{ij}\R_{ij}^T \bm\omega_{ik}^{\wedge} \right) \\
		&\qquad\qquad = \frac{1}{2} \tr \left( -\R_{ij} {\bm\omega_{ji}^{\wedge}}^T \R_{ij}^T \bm\omega_{ik}^{\wedge} \right) \\
		&\qquad\qquad = \frac{1}{2} \tr \left( {\bm\omega_{ij}^{\wedge}}^T \bm\omega_{ik}^{\wedge} \right) \\
		&\qquad\qquad = \bm\omega_{ij}^T\bm\omega_{ik}
		\, .
	\end{aligned} \end{equation}
\end{proof}

\begin{proof}[Proof of Lemma~\ref{lem:relative_control_product}]
The product $\bar\R^T \dot{\bar\R}$ is given by, 
\begin{equation*} \begin{aligned}
	\bar\R^T \dot{\bar\R} & = (\R_i^T\R_j \tilde\R_{ij}^T)^{-1}\frac{d}{dt}(\R_i^T\R_j \tilde\R_{ij}^T) \\
	&= \tilde\R_{ij}\R_j^T\R_i \left( \dot{\R}_i^T\R_j \tilde\R_{ij}^T + \R_i^T\dot{\R}_j \tilde\R_{ij}^T \right) \\
	&=  \tilde\R_{ij}\R_{ij}^T \dot{\R}_i^T\R_j \tilde\R_{ij}^T + \tilde\R_{ij}\R_j^T\dot{\R}_j \tilde\R_{ij}^T
	\, .
\end{aligned} \end{equation*}
The time derivative of a rotation matrix $\R_i$ is defined by the rotation estimation system~\eqref{eq:rotation_estimation_system} and substituting this into the previous equation yields the stated result, 
\begin{equation*} \begin{aligned}
	\dot{\bm\omega}_{ij}^{\wedge} 
	&= \tilde\R_{ij}\R_{ij}^T \left(\R_i \bm\omega_i^{\wedge}\right)^T \R_j \tilde\R_{ij}^T + \tilde\R_{ij}\R_j^T \left(\R_j \bm\omega_j^{\wedge}\right) \tilde\R_{ij}^T \\
	&= \tilde\R_{ij}\R_{ij}^T {\bm\omega_i^{\wedge}}^T \R_{ij} \tilde\R_{ij}^T + \tilde\R_{ij} \bm\omega_j^{\wedge} \tilde\R_{ij}^T \\
	&= - \tilde\R_{ij}\R_{ij}^T \bm\omega_i^{\wedge} \R_{ij}\tilde\R_{ij}^T + \tilde\R_{ij}\bm\omega_j^{\wedge}\tilde\R_{ij}^T
	\, .
\end{aligned} \end{equation*}
Additionally, this expression is indeed a skew-symmetric matrix since it is the sum of two elements in $so(3)$ from the property of logarithm products~\eqref{eq:log_products_result}. 
\end{proof}



\bibliographystyle{IEEEtran}
\bibliography{IEEEfull,00_references.bib}

\end{document}